\documentclass[10pt,twocolumn,letterpaper]{article}

\usepackage[pagenumbers]{cvpr} 

\usepackage{multirow}
\usepackage{colortbl}
\usepackage{lipsum}
\usepackage{bm}
\usepackage{subcaption}

\usepackage{amsmath,amssymb}
\usepackage{algpseudocode}
\usepackage{algorithm}
\usepackage{pifont}

\usepackage{xcolor}
\definecolor{cvprgreen}{rgb}{0.0,0.55,0.0}
\definecolor{cvprred}{rgb}{0.8,0.1,0.1}
\newcommand{\cvprcheck}{\textcolor{cvprred}{\checkmark}}
\newcommand{\cvprxmark}{\textcolor{cvprgreen}{\ding{55}}}

\usepackage{amsthm}

\newtheorem{proposition}{Proposition}

\usepackage{array}
\newcolumntype{C}{>{\centering\arraybackslash}p{0.08\textwidth}}

\usepackage{tabularx}
\newcolumntype{S}{>{\centering\arraybackslash}p{1.3cm}}  
\newcolumntype{T}{>{\centering\arraybackslash}p{1.9cm}}  
\newcolumntype{D}{>{\centering\arraybackslash}X}         

\definecolor{cvprblue}{rgb}{0.21,0.49,0.74}
\usepackage[pagebackref,breaklinks,colorlinks,allcolors=cvprblue]{hyperref}

\def\paperID{}
\def\confName{CVPR}
\def\confYear{2026}

\title{Label-Free Cross-Task LoRA Merging with Null-Space Compression}

\author{
    Wonyoung Lee$^*$ \\
    KAIST \\
    {\tt\small wylee@kaist.ac.kr}
    \and
    Wooseong Jeong$^*$ \\
    KAIST \\
    {\tt\small stk14570@kaist.ac.kr}
    \and
    Kuk-Jin Yoon \\
    KAIST \\
    {\tt\small kjyoon@kaist.ac.kr}
}

\begin{document}
\maketitle

\renewcommand{\thefootnote}{\fnsymbol{footnote}}
\footnotetext[1]{Equal contribution to this work.}

\begin{abstract}
	Model merging combines independently fine-tuned checkpoints without joint multi-task training.
	In the era of foundation-model, fine-tuning with Low-Rank Adaptation (LoRA) is prevalent, making LoRA merging a promising target. Existing approaches can work in homogeneous settings where all target tasks are classification but often fail when tasks span classification and regression. Approaches using entropy-based surrogates do not apply to regression and are costly for large language models due to long token sequences. We introduce \emph{Null-Space Compression (NSC) Merging}, a label-free, output-agnostic method that sets merge weights from adapter geometry. Our key observation is that during LoRA finetuning the down-projection factor \(A\) in \(\Delta W = BA\) compresses its null space, and the compression correlates with performance. NSC uses this as an optimization signal for merging that can generalize across classification, regression, and sequence generation. NSC achieves state-of-the-art performance across twenty heterogeneous vision tasks with balanced gains where prior methods overfit subsets of tasks. It also outperforms baselines on six NLI benchmarks and on vision–language evaluations for VQA and image captioning, demonstrating scalability and effectiveness.
    Our code is available at \url{https://github.com/wonyoung01/nsc_merging}

\end{abstract}

\section{Introduction}
\label{sec:intro}

Modern deep learning models, including vision transformers~\cite{dosovitskiy2020image} and
large foundation models such as LLMs~\cite{dubey2024llama} and
VLMs~\cite{liu2023visual, liu2024improved}, achieve strong performance and exhibit broad
generalization to unseen tasks.
Despite their scale, these models still require task-specific fine-tuning when the target
distribution or objective departs from what was covered during pretraining.
Conventional full fine-tuning, however, is computationally and memory intensive, limiting its
practicality in many real-world settings.

This challenge has motivated the development of parameter-efficient fine-tuning (PEFT)
methods~\cite{houlsby2019parameter,peft}.
Among them, Low-Rank Adaptation (LoRA)~\cite{hu2022lora} has become a widely adopted strategy
for adapting large models while keeping base weights frozen.
LoRA is advantageous for training and deployment, as it substantially reduces
the number of trainable parameters during fine-tuning and produces compact adapter modules
that are easy to distribute and share through open-source model hubs.

Beyond single-task adaptation, modern systems are expected to handle multiple downstream tasks. The traditional solution is multi-task learning~\cite{yu2020gradient, liu2021conflict, liu2023famo, jeong2024quantifying, jeong2025selective, jeong2025resolving, jeong2025synchronizing, jeong2026stabilizing}, which requires a unified training pipeline and a joint dataset with aligned per-task labels. This is costly in curation and computation, and often infeasible when labels are proprietary or restricted by license. It also makes it difficult to leverage numerous task-specialized models that are scattered across open model hubs, because those checkpoints cannot be reused directly without centralized retraining. These constraints motivate model merging, where independently finetuned checkpoints are combined without revisiting labeled data. In the foundation-model era, model merging is especially effective with LoRA adapters. One can merge or compose adapters trained on diverse tasks, retrieved independently from model hubs, to synthesize a single model that inherits their strengths while preserving the benefits of lightweight storage and easy distribution.

Conventional model merging has largely focused on full model weights \cite{ilharco2022editing, jin2022dataless, yadav2023ties, yang2023adamerging, yu2024language}. Task Arithmetic \cite{ilharco2022editing} constructs task vectors as differences between fine-tuned and pretrained checkpoints, then adds them with a scale chosen on a small validation split. In contrast, LoRA-aware merging \cite{stoica2025knots, zhao2025loralego} operates directly on the low-rank adapters rather than the full weights. Working in the low-rank space is memory and compute efficient and it performs markedly better in LoRA settings, where conventional full-weight mergers can degrade due to strong subspace misalignment between adapters \citep{stoica2025knots}. Given that LoRA fine-tuning is now standard for LLMs and VLMs, LoRA-targeted merging is a promising direction.

Another advantage of LoRA merging is that it can use gradient signals to set merge weights across checkpoints, which is prohibitive at the scale of LLMs or VLMs when operating on full weights. AdaMerging~\cite{yang2023adamerging} is a representative gradient-based method that estimates merge weights without labels by minimizing output entropy as a surrogate objective. This surrogate is standard in test-time adaptation \cite{wang_tent_2020} and has been adopted by AdaMerging and subsequent work \cite{tang_merging_2024, wei_modeling_2025, chen2024pareto}, yielding strong gains on classification. However, entropy-based approaches face fundamental limitations in broader tasks and domains. First, entropy does not apply to regression problems such as depth estimation or surface normal prediction. Second, for LLMs and VLMs, entropy must be computed at each token prediction, so cost grows with the number of generated tokens and quickly becomes a bottleneck for \emph{entropy-based} merging.

To address these issues, we introduce \emph{Null-Space Compression (NSC) Merging}. Our key observation is that during LoRA fine-tuning the down-projection factor $\bm A$ in $\Delta \bm W = \bm B \bm A$ defines a projection onto a low-dimensional subspace of the input features, and the corresponding null-space ratio, which is the proportion of activation discarded by this projection, decreases over training. This decrease is closely correlated with task performance. We use this null-space compression as a label-free signal to determine merging weights across adapters. Unlike entropy-based surrogates, NSC applies naturally to both classification and regression and remains efficient for long sequence generation.  NSC Merging demonstrates superior performance across heterogeneous scenarios that mix classification and regression, and we validate it on 20 vision tasks, natural language inference for LLMs, and Vision Question Answering (VQA) and image captioning for VLMs.

Our main contribution are as follows:
\begin{itemize}[leftmargin=*, itemsep=2pt, topsep=2pt]
	\item We observe LoRA finetuning systematically compresses the null space of the down-projection factor \(\bm A\), and the size of null space is closely correlated with the peformance of LoRA adapted model, that can be used as learning signal.
	\item Based on the observations, we suggest Null-Space Compression (NSC) Merging using a effective gradient-based adaptation signal for model merging which can be used for various scenario where entropy-based methods cannot be applied or inefficient.
	\item We validate our methods with extensive experiments including various models and benchmarks including conventional vision encoder, LLM, VLM and shows state-of-the-art performance compared to previous model merging baselines.
\end{itemize}

\section{Related Work}
\label{sec:related}

\textbf{Model Merging.}
Model merging combines knowledge from models that were fine-tuned separately, avoiding joint multi-task training and reducing computational cost while aiming to preserve task performance \citep{yang2024model}.
Representative methods include \emph{Task Arithmetic}, which adds or subtracts task vectors \citep{ilharco2022editing}, \emph{RegMean}, which uses inner-product statistics of layer inputs to enable data-free fusion \citep{jin2022dataless}, \emph{TIES}, which resolves sign conflicts and prunes small updates \citep{yadav2023ties}, and \emph{DARE}, which sparsifies parameter deltas \citep{yu2024language}.
\emph{AdaMerging} estimates layer-wise merging coefficients by minimizing an entropy-based objective to improve multi-task utility \citep{yang2023adamerging, tang_merging_2024, wei_modeling_2025, chen2024pareto}.
When models are trained on different data or initialized differently, permutation alignment helps maintain linear mode connectivity \citep{entezari2021role,ainsworth2022git}.
\emph{ZipIt} aligns intermediate features to re-basin networks and match neurons across models, enabling training-free layer alignment \citep{stoica2023zipit}.
Uncertainty-aware strategies mitigate mismatch through Fisher-weighted averaging \citep{matena2022merging} and through uncertainty or gradient matching \citep{daheim2023model}. Pareto merging \citep{chen2024pareto} motivate multi-objective treatments considering preferences.

\vspace{2pt}
\noindent\textbf{LoRA Merging.}
Low-Rank Adaptation (LoRA) \citep{hu2022lora} is a standard fine-tuning mechanism for large networks, including foundation models such as LLMs \citep{dubey2024llama}.
Conventional full-parameter merging methods \citep{jin2022dataless,ilharco2022editing,yadav2023ties,yu2024language} transfer poorly to LoRA adapters, which has motivated LoRA-specific approaches \citep{stoica2025knots,zhao2025loralego,tang2025lora}.
\citet{stoica2025knots} report that LoRA-updated models exhibit weaker cross-model representation alignment than full-rank fine-tunes and propose \textsc{KnOTS}, which concatenates adapter updates and applies SVD to align them in a shared subspace before merging principal components.
\citet{zhao2025loralego} introduce minimal semantic units and use clustering to assemble a merged adapter with adjustable effective rank.
Compositions of multiple LoRAs have also been explored for image generation, including multi-LoRA composition and concept mixing \citep{zhong2024multi,gandikota2024concept,zhuang2025timestep}.

Conventional model merging underperforms on LoRA-adapted models, whereas LoRA-targeted merging performs better. Given the prevalence of PEFT in foundation models, LoRA-targeted merging is a promising direction. In the LoRA regime, gradient-based methods such as AdaMerging are feasible in compute and memory, unlike merging full-rank fine-tuned weights. However, AdaMerging relies on entropy minimization, which does not apply to regression and scales poorly to LLMs and VLMs. In this paper, we propose a new gradient-based method that is efficient, handles both classification and regression, and scales to LLMs and VLMs where prior approaches are inefficient. Further discussion of prior work and our approach appears in \Cref{append:more_related_work} of the supplementary material.

\section{Method}
\label{method}

To extend model merging to broader tasks and domains, we propose a gradient-based merging algorithm, termed \emph{Null-Space Compression (NSC) merging}.
The optimization objective stems from our key observation about the
\emph{dynamics of LoRA layers during fine-tuning}, which we refer to as \emph{null-space compression}.
Our method is designed to be task-agnostic: it optimizes layer-wise merging coefficients using only \emph{unlabeled data} and \emph{structural information} available from the LoRA adapter.
Before delving into this phenomenon, we outline the goal of our approach.

\subsection{Preliminaries}

Consider a pretrained base model with $L$ layers and parameters $\{\bm W_0^{\ell}\}_{\ell=1}^{L}$, where $\bm W_0^{\ell}$ denotes the parameter tensor at layer $\ell$. Let there be $K$ downstream tasks indexed by $k\in\{1,\dots,K\}$. For each task $k$, let $\{\bm W_k^{\ell}\}_{\ell=1}^{L}$ be the parameters of the model obtained by fine-tuning on task $k$ starting from the base model.
The difference $\Delta \bm W_k^\ell = \bm W_k^\ell - \bm W_0^\ell$ is referred to as a \emph{task vector}~\cite{ilharco2022editing}.
Model merging~\citep{yang2024model} aims to combine these task vectors into a single model, allowing efficient multi-task adaptation and knowledge integration.

With the rise of large-scale foundation models~\citep{dubey2024llama, liu2023visual}, full fine-tuning is often infeasible, and parameter-efficient fine-tuning with LoRA~\citep{hu2022lora} has been widely adopted for its simplicity and efficiency. Consequently, principled merging strategies tailored to LoRA are essential. In this paper, we propose a merging strategy for LoRA-fine-tuned models and validate it across diverse tasks and models.
We denote the task-specific LoRA as $\{\Delta\bm W_{k}^{\ell}\}_{\ell \in \mathcal{J}}$,
where $\mathcal{J} \subseteq \{1, \dots, L\}$ is the subset of layers equipped with adapters.
Each task vector $\Delta\bm W_{k}^{\ell}$ is expressed as
\begin{equation}
	\Delta\bm W_{k}^{\ell} = \bm B_{k}^{\ell} \bm A_{k}^{\ell},
	\quad
	\bm B_{k}^{\ell} \in \mathbb{R}^{d_{\text{out}}^{\ell} \times r_k},\;
	\bm A_{k}^{\ell} \in \mathbb{R}^{d_{\text{in}}^{\ell} \times r_k}
	\label{eq:lora}
\end{equation}
where $r_k$ denotes the adapter rank for the $k$-th task and typically satisfies
$r_k \ll \min(d_{\text{in}}^{\ell}, d_{\text{out}}^{\ell})$.

Given layer-wise coefficients $\lambda_{k}^{\ell}$ for task $k$, the merged update is defined as
\begin{equation}
	\Delta\bm W_{\text{merge}}^{\ell} =
	\begin{cases}
		\displaystyle \sum_{k=1}^{K} \lambda_{k}^{\ell} \, \Delta\bm W_{k}^{\ell}, & \ell \in \mathcal{J},    \\[6pt]
		\mathbf{0},                                                                & \ell \notin \mathcal{J}.
	\end{cases}
	\label{eq:merged_update}
\end{equation}
Our goal is to learn optimal coefficients $\{\lambda_{k}^{\ell}\}$ such that the merged model
$\bm \Theta_\text{merge} = \{\bm W_{0}^{\ell} + \Delta\bm W_{\text{merge}}^{\ell}\}_{\ell=1}^{L}$ maximizes overall performance across tasks,
balancing shared representations and task-specific specialization.

\subsection{Null-space compression in LoRA fine-tuning}

\begin{figure}[t]
	\centering
	\begin{subfigure}{\linewidth}
		\centering
		\includegraphics[width=\linewidth]{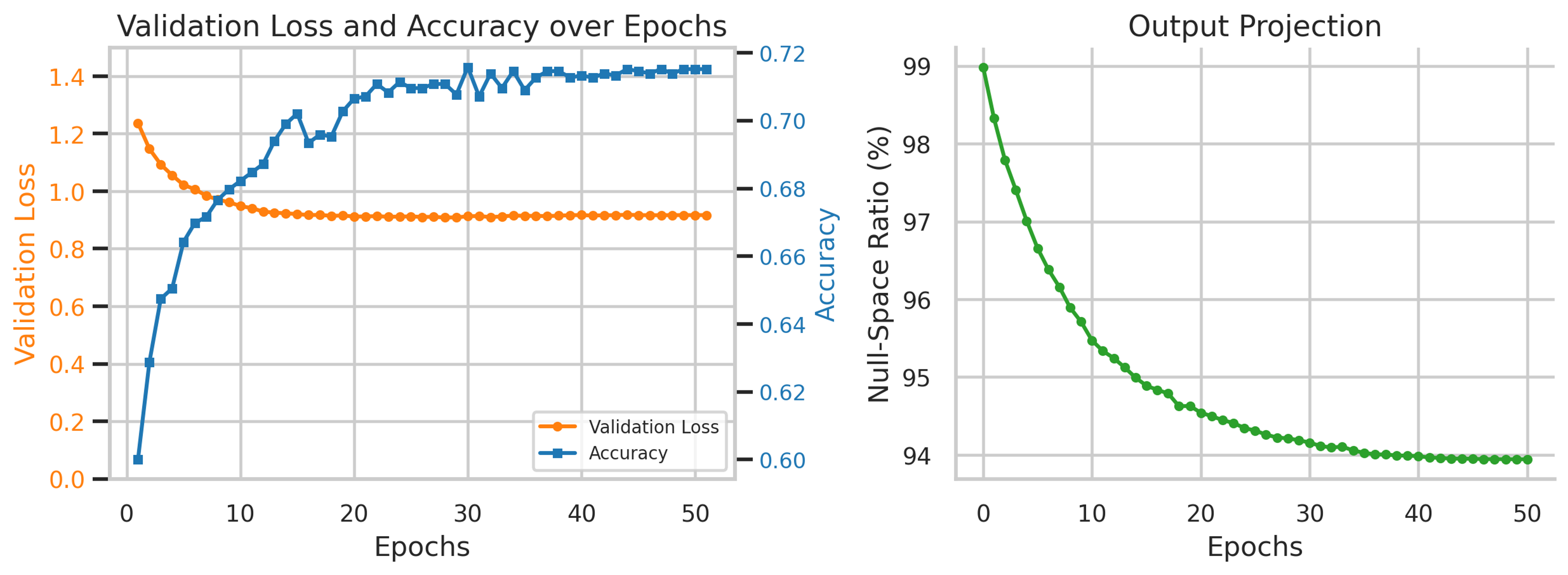}
		\caption{Image Classification (Classification Task)}
		\label{subfig:cls_nsc}
	\end{subfigure}

	\vspace{0.6em} 

	\begin{subfigure}{\linewidth}
		\centering
		\includegraphics[width=\linewidth]{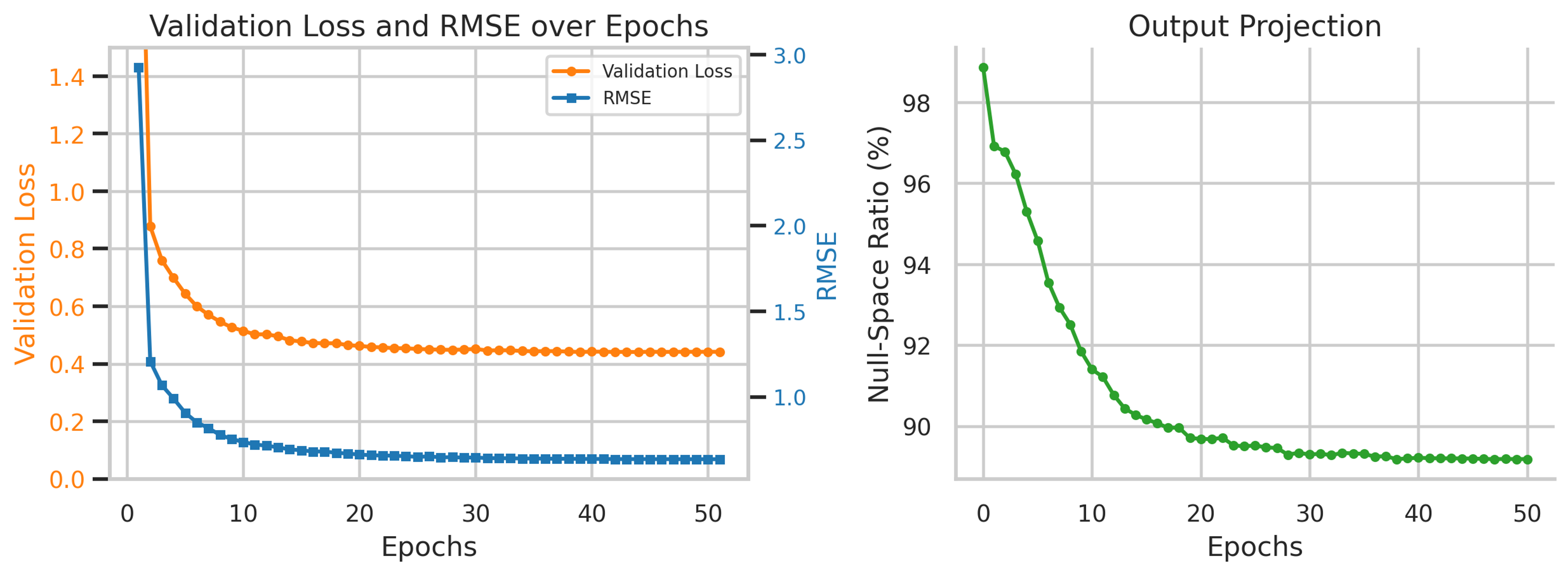}
		\caption{Depth Estimation (Regression Task)}
		\label{subfig:depth_nsc}
	\end{subfigure}

	\caption{Visualization of validation loss, task performance, and the null-space ratio for an output projection layer during LoRA fine-tuning on (a) image classification and (b) depth estimation.}
	\label{fig:nsc_finetuning}
	\vspace{-5pt}
\end{figure}

Considering the structure of LoRA, we denote $\bm A_{k}^{\ell}$ in \cref{eq:lora} as down-projection matrix. This down-projection matrix inherently project incoming activations into a reduced subspace, and hence discarding components that fall into the \emph{null space}. We refer to the proportion of activation suppressed by this projection as the \emph{null-space ratio}. Specifically, we define the null-space ratio of an adapter weight $\Delta \bm W_k^{\ell} = \bm B_k^{\ell} \bm A_k^{\ell}$ with input feature $\bm z$ as
\begin{equation}
	\omega_k^{\ell}(\bm z)
	=
	\frac{
	\left\|
	\mathrm{Proj}_{\mathcal{N}(\bm A_k^{\ell})}(\bm z)
	\right\|_2
	}{
	\left\|
	\bm z
	\right\|_2
	},
	\label{eq:null_ratio}
\end{equation}
where $\mathrm{Proj}_{\mathcal{N}(\bm A_k^{\ell})}$ denotes the orthogonal
projection operator onto the null space of $\bm A_k^{\ell}$.
This ratio measures the extent to which the input feature is discarded by the adapter:
a higher $\omega_k^{\ell}(\bm z)$ indicates that a larger portion of the activation lies in the null space, implying that less information is propagated to subsequent layers.

\vspace{2pt}
\noindent\textbf{Compression dynamics during fine-tuning.} To quantify layer-wise null-space ratio,
we compute the expectation $\mathbb{E}_{\bm x}[\omega_k^\ell(\bm z^\ell(\bm x; \bm \Theta))]$ over unlabeled samples,
where $\bm z^{\ell}(\bm x; \bm \Theta)$ represents the \emph{incoming} activation to the $\ell$-th layer of a model parameterized by $\bm \Theta$ for an input batch $\bm x$. \Cref{fig:nsc_finetuning} tracks how the null-space ratio evolves during LoRA fine-tuning. For \emph{classification} (\cref{subfig:cls_nsc}), we fine-tune CLIP~\citep{radford2021learning} with a ViT-B/32 backbone and report validation loss, accuracy, and the layer-wise null-space ratio. For \emph{regression} (\cref{subfig:depth_nsc}), we fine-tune a ViT-B/16~\citep{dosovitskiy2020image} backbone with a lightweight decoder on depth estimation and report validation loss, RMSE, and the same ratio. In both settings, LoRA adapters of rank 16 are attached to the query, key, value, and output projections in each self-attention block, and training runs for 50 epochs.

As training progresses, the null-space ratio consistently decreases, indicating that a smaller fraction of the activation lies in the adapter’s null space.
Both models use a latent dimension of 768, while the LoRA rank is 16, so the adapter subspace spans about 2.1 percent of the feature space.
Given this small coverage, the observed decrease is substantial, showing that more of the representation is being captured within the adapter subspace.
Empirically, a lower null-space ratio is associated with higher task performance, exhibiting a strong inverse correlation across datasets and layers.
The same pattern holds for both classification and regression despite differences in objectives and data domains.
This motivates us to use the ratio as a task-agnostic signal for merging that generalizes across settings.
Similar behavior appears in other LoRA layers and transformer blocks (see \cref{append:more_analysis_nsc_finetune} in the supplementary material).

\begin{figure}[t]
	\centering
	\includegraphics[width=0.47\textwidth]{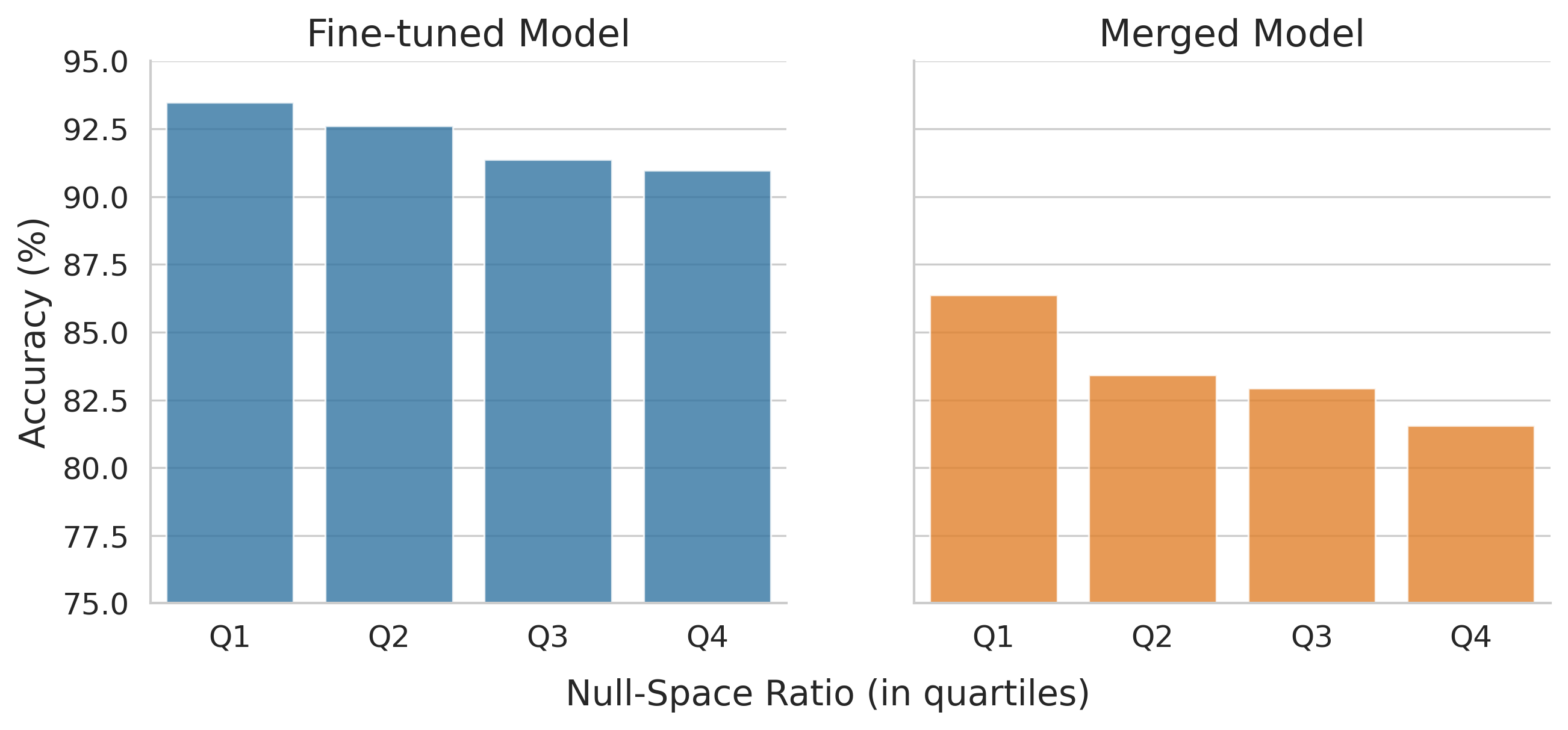}
	\caption{Classification accuracy versus the \emph{null-space ratio}. Accuracy is averaged within quartiles of the ratio across tasks. Left shows fine-tuned experts, right shows the merged model.}
	\label{fig:nsc_acc_quantile}
	\vspace{-5pt}
\end{figure}

\vspace{2pt}
\noindent\textbf{Performance correlation after fine-tuning.} The null-space ratio is closely related to model performance after fine-tuning. \Cref{fig:nsc_acc_quantile} shows its relationship to classification accuracy. For each task, we partition test samples into four quantiles by their ratio and report the average accuracy within each quantile across tasks. In the \emph{fine-tuned model} (left), samples with lower ratios achieve higher accuracy. A similar pattern holds in the \emph{merged model} (right), where performance remains inversely correlated with the ratio in a merged model with task arithmetic~\cite{ilharco2022editing}. Even after merging, samples that interact more strongly with their adapters (lower null-space ratios) show the same inverse correlation.

\subsection{Null-Space Compression Merging}
Motivated by the observed inverse correlation between null-space compression and task performance, we use the null-space ratio as a label-free learning signal to estimate merge coefficients and propose \emph{Null-Space Compression (NSC) Merging}. Unlike applying gradient-based signals when merging full-rank fine-tuned networks, which is memory prohibitive, LoRA merging updates only lightweight adapters, so the extra memory and compute are modest, making the approach feasible in practice.

We compute the mean null-space ratio across all LoRA-equipped layers during inference as
\begin{equation}
	\Omega_k(\bm x; \bm \Theta)
	=
	\frac{1}{|\mathcal{J}|}
	\sum_{\ell \in \mathcal{J}}
	\omega_k^{\ell}\!\left(
	\bm z^{\ell}(\bm x; \bm \Theta)
	\right),
	\label{eq:mean_null_ratio}
\end{equation}
and estimate its expectation $\mathbb{E}_{\bm x}[\Omega_k(\bm x; \bm \Theta)]$ over unlabeled samples. Finally, we propose objectives for learning the merging coefficients.

\vspace{2pt}
\noindent\textbf{Optimization Objective.}
We minimize the null-space ratio as a task-agnostic surrogate loss for optimizing model merging coefficients. Precisely, the NSC objective is formulated as
\begin{equation}
	\begin{aligned}
		 & \min_{\{\lambda_k^{\ell}\}}
		\frac{1}{K} \sum_{k=1}^K \left(
		\mathop{\mathbb{E}}\limits_{
			\substack{
				\bm{x} \sim \mathcal{D}_k
			}
		}
		\Big[
			\Omega_k\big(\bm{x}; \bm{\Theta}_{\text{merge}}\big)
		\Big] \right)                  \\[4pt]
		 & \text{s.t.} \quad
		\bm{\Theta}_{\text{merge}}
		=
		\big\{
		\bm{W}_0^{\ell}
		+
		\textstyle\sum_{k}
		\lambda_k^{\ell}
		\bm{B}_k^{\ell}
		\bm{A}_k^{\ell}
		\big\}_{\ell=1}^L
	\end{aligned}
	\label{eq:nsc_merged}
\end{equation}
where $\mathcal{D}_k$ denotes an unlabeled validation set for task $k$.

\vspace{2pt}
\noindent\textbf{Fast NSC: Caching the Adapter Gram-Inverse.}
Before optimizing the merging coefficients in the NSC objective in~\cref{eq:nsc_merged}, we pre-compute the parameters required for evaluating the null-space ratio at each LoRA-equipped layer. For a LoRA update $\Delta\bm W_k \!=\! \bm B_k \bm A_k$, the null-space ratio of a feature $\bm z$ can be expressed as
\begin{equation}
	\omega_k(\bm z)
	=
	\sqrt{
	1 -
	\frac{
	\bm z^{\top} \bm A_k^{\top} (\bm A_k \bm A_k^{\top})^{-1} \bm A_k \bm z
	}{
	\|\bm z\|_2^2
	}
	}.
	\label{eq:null_ratio_simple}
\end{equation}
This expression is equivalent to \cref{eq:null_ratio}. For clarity we write it using the down-projection matrix, and for brevity we omit the layer index $\ell$. Since $\bm z$ and $\bm A_k \bm z$ are computed during inference, the only additional term needed is $(\bm A_k \bm A_k^{\top})^{-1}$, a small square matrix of dimension equal to the LoRA rank. This avoids explicit construction of the full null-space projection matrix, greatly reducing both memory usage and computational overhead. As a result, the proposed computation enables short preparation time and efficient optimization of the merging coefficients. A detailed derivation of~\cref{eq:null_ratio_simple} is in \cref{append:gram_inverse_null_space_ratio} of supplementary material, and efficiency is analyzed in~\cref{ablation:cost}.

\vspace{2pt}
\noindent\textbf{Target Layers for Objective Computation.}
Computing the null-space ratio at every LoRA-equipped layer during inference incurs computational and memory overhead.
Meanwhile, optimization based on the null-space ratio from later transformer blocks and projection layers still influences the merging coefficients of earlier layers.
Hence, we analyze effective target layers for computing the NSC objective.
Formally, rather than computing $\Omega_k$ across the entire set $\mathcal{J}$, restricting computation to a smaller subset $\mathcal{J}_\text{tgt} \subseteq \mathcal{J}$ offers a better trade-off between efficiency and performance.
In practice, we use only the last quarter of transformer blocks for computing the objective. Experiments supporting this choice are provided in \cref{subsec:ablation}.

\vspace{2pt}
\noindent\textbf{Advantages over Entropy Minimization.}
NSC merging extends gradient-based model merging by relying solely on the structural information of LoRA parameters, making it label-free and suitable for heterogeneous tasks where entropy-based objectives are ill-defined.
For LLMs and VLMs performing next-token prediction, entropy objectives must be evaluated at every generated token, causing cost to scale with sequence length and makes the algorithm not scalable.
In contrast, NSC is \emph{input}-oriented: it estimates merge weights from initial input prompts or image tokens without relying on \emph{output} logits, keeping its cost independent of generated sequence length.
Its ability to adapt using only input IDs is demonstrated in~\cref{subsubsec:vlm_result}, and a detailed cost analysis in~\cref{ablation:cost}.
Overall process of NSC merging is summarized in~\cref{alg:nsc_merging_pseudocode}.

\begin{algorithm}[t]
	\caption{NSC Merging}
	\label{alg:nsc_merging_pseudocode}
	\begin{algorithmic}[1]
		\Require Pretrained $\{\bm W_0^{\ell}\}_{\ell=1}^{L}$; LoRA $\{(\bm B_k^{\ell},\bm A_k^{\ell})\}$ for $k{=}1{:}K$, unlabeled validation set $\{\mathcal{D}_k\}$; target layers $\mathcal{J}_{\mathrm{tgt}}\!\subseteq\!\mathcal{J}$; steps $T$; batch size $b$; learning rate $\eta$
		\Ensure Coefficients $\{\lambda_k^{\ell}\}$, merged model $\bm\Theta_{\mathrm{merge}} $

		\Statex \textbf{Step 1: Prepare adapter Gram-inverse}
		\For{$k=1$ to $K$}
		\For{$\ell \in \mathcal{J}_{\mathrm{tgt}}$}
		\State Compute and cache $\left(\bm A_k^{\ell}\bm A_k^{\ell\top}\right)^{-1}$
		\EndFor
		\EndFor

		\Statex \textbf{Step 2: Optimize coefficients}
		\State Initialize $\lambda_k^{\ell}\quad \forall\,k,\ell \in \mathcal{J}$
		\For{$t=1$ to $T$}
		\State $\bm\Theta_{\mathrm{merge}} \gets \big\{\bm{W}_0^{\ell} + \textstyle\sum_{k} \lambda_k^{\ell} \bm{B}_k^{\ell} \bm{A}_k^{\ell}\big\}_{\ell=1}^L$
		\For{$k=1$ to $K$}
		\State Sample $\{\bm x_i\}_{i=1}^{b} \sim \mathcal{D}_k$ \Comment{Unlabeled mini-batch}
		\State $\widehat{\mathcal{L}}_k \gets \frac{1}{b}\sum_{i=1}^{b}\Omega_k(\bm x_i;\bm\Theta_{\mathrm{merge}})$
		\EndFor
		\State $\widehat{\mathcal{L}} \gets \frac{1}{K}\sum_{k=1}^{K}\widehat{\mathcal{L}}_k$
		\State $\lambda_k^{\ell} \gets \lambda_k^{\ell} - \eta\,\frac{\partial \widehat{\mathcal{L}}}{\partial \lambda_k^{\ell}} \quad \forall\,k,\ell \in \mathcal{J}$
		\EndFor
		\State \Return $\{\lambda_k^{\ell}\}, \bm\Theta_{\mathrm{merge}}$
	\end{algorithmic}
\end{algorithm}

\section{Experiments}
\subsection{Experimental Setup}
\noindent
\textbf{Datasets and Tasks.}
We evaluate across three settings: a heterogeneous vision setting with ViT-based multi-task learning, an LLM sequence classification setting, and a VLM vision–language setting. For vision, we use three multi-task datasets spanning indoor and outdoor scenes.
NYUD-v2~\citep{silberman2012indoor} provides four tasks: depth estimation, semantic segmentation, surface normal prediction, and edge detection. PASCAL-Context~\citep{mottaghi2014role} includes five tasks: semantic segmentation, human parts estimation, saliency estimation, surface normal prediction, and edge detection. Taskonomy~\citep{zamir2018taskonomy} covers eleven tasks: Depth Euclidean (DE), Depth Z-buffer (DZ), Edge Texture (ET), Keypoints 2D (K2), Keypoints 3D (K3), Normal (N), Principal Curvature (C), Reshading (R), Segment Unsup2D (S2), and Segment Unsup2.5D (S2.5).

For large language models (LLMs), we use \textsc{LLaMA-3-8B} \cite{dubey2024llama} on six natural language inference (NLI) datasets: MNLI, QNLI, SNLI, RTE, SICK, and SciTail. For vision-language models (VLMs), we employ \textsc{LLaVA-1.5-7B} \cite{liu2023visual, liu2024improved} across a diverse set of multimodal tasks, ranging from word-level visual question answering to sentence-level generation. Specifically, we evaluate on VizWiz, IconQA, ChartQA, DocVQA, COCO, and Flickr30k.
Further details on the datasets and data splits are provided in supple~\cref{append:datasets}.

\begin{table*}[t]
	\vspace{-10pt}
	\caption{Merging results on NYUD-v2, PASCAL-Context, and Taskonomy using ViT-B. The fine-tuned rows report per-task absolute performance. For each merging baseline, we report the normalized performance (\%) relative to the corresponding fine-tuned score. The rightmost \textit{Avg} column is the mean of the normalized scores across tasks.}
	\vspace{-5pt}
	\centering
	\renewcommand\arraystretch{1.05}
	\setlength{\tabcolsep}{2pt}
	\resizebox{0.99\textwidth}{!}{
		\begin{tabular}{l|
				cccc| 
				ccccc| 
				ccccccccccc| 
				c} 
			\toprule
			 & \multicolumn{4}{c|}{\textbf{NYUD-v2 (4)}}
			 & \multicolumn{5}{c|}{\textbf{PASCAL-Context (5)}}
			 & \multicolumn{11}{c|}{\textbf{Taskonomy (11)}}
			 & \multirow{3}{*}{\textbf{Avg}}                                                                                                                                                                                                 \\
			\cmidrule(lr){2-5}\cmidrule(lr){6-10}\cmidrule(lr){11-21}
			\multirow{2}{*}{\textbf{Method}}
			 & Depth                                            & Semseg & Normal   & Edge
			 & Semseg                                           & Parts  & Saliency & Normal & Edge
			 & DE                                               & DZ     & EO       & ET     & K2    & K3    & N     & C     & R     & S2    & S2.5
			 &                                                                                                                                                                                                                               \\
			 & \multicolumn{1}{c}{RMSE $\downarrow$}
			 & \multicolumn{1}{c}{mIoU $\uparrow$}
			 & \multicolumn{1}{c}{Mean $\downarrow$}
			 & \multicolumn{1}{c|}{L1 $\uparrow$}
			 & \multicolumn{1}{c}{mIoU $\uparrow$}
			 & \multicolumn{1}{c}{mIoU $\uparrow$}
			 & \multicolumn{1}{c}{mIoU $\uparrow$}
			 & \multicolumn{1}{c}{Mean $\downarrow$}
			 & \multicolumn{1}{c|}{L1 $\uparrow$}
			 & \multicolumn{1}{c}{L1 $\downarrow$}
			 & \multicolumn{1}{c}{L1 $\downarrow$}
			 & \multicolumn{1}{c}{L1 $\downarrow$}
			 & \multicolumn{1}{c}{L1 $\downarrow$}
			 & \multicolumn{1}{c}{L1 $\downarrow$}
			 & \multicolumn{1}{c}{L1 $\downarrow$}
			 & \multicolumn{1}{c}{L1 $\downarrow$}
			 & \multicolumn{1}{c}{RMSE $\downarrow$}
			 & \multicolumn{1}{c}{L1 $\downarrow$}
			 & \multicolumn{1}{c}{L1 $\downarrow$}
			 & \multicolumn{1}{c|}{L1 $\downarrow$}
			 &                                                                                                                                                                                                                               \\
			\midrule
			\multicolumn{22}{c}{\textit{Finetuned performance}}                                                                                                                                                                              \\ \cmidrule(lr){1-22}
			Finetuned
			 & 0.657                                            & 37.66  & 25.98    & 0.051
			 & 70.07                                            & 54.77  & 80.00    & 18.36  & 0.046
			 & 0.016                                            & 0.016  & 0.101    & 0.171  & 0.162 & 0.082 & 0.217 & 0.710 & 0.136 & 0.170 & 0.144
			 & --                                                                                                                                                                                                                            \\
			\midrule
			\multicolumn{22}{c}{\textit{Merged models, normalized to finetuned (\%)}}                                                                                                                                                        \\ \cmidrule(lr){1-22}
			TA~\cite{ilharco2022editing}
			 & 24.0                                             & 1.3    & 54.9     & 105.2  & 4.7   & 20.6  & 36.4  & 62.5  & 104.3 & 100.1 & 99.8  & 101.7 & 104.6 & 106.4 & 102.2 & 100.2 & 103.6 & 106.2 & 107.4 & 98.8  & 77.2          \\

			TIES~\cite{yadav2023ties}
			 & 28.5                                             & 1.3    & 55.2     & 105.1  & 4.7   & 20.6  & 34.6  & 66.1  & 104.3 & 102.5 & 102.0 & 100.6 & 103.7 & 105.4 & 99.5  & 100.1 & 102.4 & 105.2 & 104.3 & 99.5  & 77.3          \\

			DARE-TIES~\cite{yu2024language}
			 & 24.3                                             & 1.3    & 54.6     & 105.2  & 4.7   & 20.6  & 32.6  & 62.2  & 104.3 & 98.1  & 97.7  & 103.3 & 104.6 & 106.4 & 104.5 & 100.1 & 103.7 & 103.9 & 104.6 & 100.6 & 76.9          \\

			SVD~\cite{tang2025lora}
			 & 6.8                                              & 1.2    & 40.6     & 105.4  & 4.8   & 20.6  & 31.2  & 49.4  & 104.3 & 98.0  & 94.4  & 103.7 & 123.4 & 146.9 & 117.8 & 101.5 & 122.2 & 143.4 & 146.7 & 96.9  & 83.0          \\

			Linear~\cite{peft}
			 & 2.6                                              & 0.5    & 30.7     & 105.2  & 0.1   & 20.5  & 28.5  & 21.4  & 102.2 & 108.3 & 105.9 & 95.9  & 120.5 & 122.1 & 117.1 & 100.4 & 98.1  & 126.4 & 75.0  & 104.9 & 74.3          \\

			KnOTS-TIES~\cite{stoica2025knots}
			 & 19.4                                             & 1.5    & 54.5     & 105.2  & 4.7   & 20.6  & 38.3  & 61.2  & 104.3 & 100.1 & 99.6  & 101.3 & 102.8 & 105.2 & 102.4 & 100.1 & 102.7 & 104.9 & 104.8 & 97.9  & 76.6          \\

			KnOTS-DARE-TIES~\cite{stoica2025knots}
			 & 18.2                                             & 1.5    & 55.4     & 105.2  & 4.7   & 20.6  & 42.5  & 59.9  & 104.3 & 103.9 & 102.8 & 97.7  & 107.8 & 107.3 & 96.1  & 100.2 & 102.4 & 103.1 & 101.7 & 100.2 & 76.8          \\

			LoRA-LEGO~\cite{zhao2025loralego}
			 & 9.3                                              & 1.4    & 52.8     & 104.3  & 4.7   & 20.6  & 36.3  & 55.4  & 104.0 & 103.9 & 104.1 & 96.7  & 106.7 & 108.7 & 99.1  & 100.2 & 103.6 & 110.4 & 102.2 & 104.5 & 76.4          \\

			EMR-Merging~\cite{huang2024emr}
			 & 21.6                                             & 1.5    & 55.3     & 103.9  & 4.7   & 20.6  & 44.9  & 58.9  & 104.0 & 98.2  & 98.2  & 102.4 & 103.0 & 104.8 & 103.2 & 100.1 & 103.4 & 105.4 & 108.1 & 98.5  & 77.0          \\

			FR-Merging~\cite{zheng2025free}
			 & 6.3                                              & 1.2    & 37.9     & 105.4  & 4.8   & 20.6  & 31.0  & 48.2  & 104.3 & 99.7  & 95.6  & 101.5 & 125.8 & 150.6 & 117.0 & 101.7 & 122.7 & 146.6 & 149.7 & 97.3  & 83.4          \\

			RobustMerge~\cite{zeng2025parameter}
			 & 37.2                                             & 76.5   & 63.4     & 100.1  & 83.1  & 73.7  & 85.7  & 73.7  & 100.6 & 100.0 & 101.0 & 102.2 & 100.6 & 100.5 & 100.6 & 100.0 & 100.2 & 99.7  & 99.1  & 99.8  & 89.9          \\

			\rowcolor[gray]{0.9}
			NSC (Ours)
			 & 45.9                                             & 85.1   & 69.6     & 100.4  & 86.7  & 76.7  & 88.7  & 80.9  & 101.7 & 100.1 & 101.0 & 102.2 & 100.8 & 100.8 & 100.7 & 100.0 & 100.4 & 100.0 & 99.6  & 99.6  & \textbf{92.0} \\
			\bottomrule
		\end{tabular}}
	\label{tab:nyud_pascal_taskonomy}
	\vspace{-5pt}
\end{table*}

\vspace{2pt}
\noindent\textbf{Baselines.}
For simplicity, we omit layer-wise index $\ell$ for this section. Let $\bm W_0$ denote the pretrained parameters, $\Delta \bm W_k$ the task vector for task $k$, and $\bm W_{\text{merge}}$ the merged parameters. For LoRA adapters we write $\Delta \bm W_k = \bm B_k \bm A_k$, and $\lambda$ denotes a global scaling coefficient. We evaluate two families of methods.
\emph{Vanilla merging} operates in full parameter space.
(\romannumeral 1) \textbf{Task Arithmetic (TA)}~\cite{ilharco2022editing} adds a scaled sum of task vectors to the pretrained weights, i.e., $\bm W_{\text{merge}}=\bm W_0+\lambda\sum_k \Delta \bm W_k$.
(\romannumeral 2) \textbf{TIES}~\cite{yadav2023ties} prunes low-magnitude coordinates, enforces sign consensus across models, then averages only the retained coordinates with scale $\lambda$.
(\romannumeral 3) \textbf{DARE-TIES}~\cite{yu2024language} applies Bernoulli sparsification with probability $p$ and rescales by $1/(1-p)$ to preserve the expectation.
(\romannumeral 4) \textbf{AdaMerging}~\cite{yang2023adamerging} learns merging coefficients by minimizing an output-entropy surrogate in the spirit of test-time adaptation~\cite{wang_tent_2020}.
One line of works exploit the low-rank structure of adapters.
(\romannumeral 5) \textbf{SVD}~\cite{tang2025lora} aggregates LoRA task vectors $\Delta \bm W_{\text{merge}}=\lambda\sum_k \bm B_k \bm A_k$, applies truncated SVD to the sum, and refactors the result back into LoRA form to recover the desired rank.
(\romannumeral 6) \textbf{Linear}~\cite{peft} performs TA directly on the factors by linearly combining $\{\bm A_k\}$ and $\{\bm B_k\}$ instead of whole task vectors.
(\romannumeral 7) \textbf{KnOTS}~\cite{stoica2025knots} computes an SVD over the concatenation of task vectors and merges by assigning the right-singular components to tasks; applying TIES or DARE-TIES to these components yields \textbf{KnOTS-TIES} and \textbf{KnOTS-DARE-TIES}.
(\romannumeral 8) \textbf{LoRA-LEGO}~\cite{zhao2025loralego} decomposes adapters into semantic units, clusters them rank-wise, and assembles a new adapter from cluster centroids.
Additionally, we evaluate (\romannumeral 9) \textbf{EMR-Merging}~\cite{huang2024emr}, (\romannumeral 10) \textbf{FR-Merging}~\cite{zheng2025free}, and (\romannumeral 11) \textbf{RobustMerge}~\cite{zeng2025parameter}.
Implementation details for all baselines appear in \cref{append:imple} of supplementary material.

\vspace{2pt}
\noindent
\textbf{Evaluation Metrics.}
We report \emph{normalized performance} relative to a finetuned reference model.
For higher-is-better metrics, a method’s score is divided by the finetuned model’s score.
For lower-is-better metrics, the finetuned score is divided by the method’s score.
Each task uses its conventional metric from the literature.
See supple~\ref{append:metrics} for detailed information.

\begin{table}[t]
	\caption{
		Per-task results on six NLI benchmarks. We merge six \textsc{LLaMA-3 8B} checkpoints, each fine-tuned with LoRA (rank 16). The top panel reports absolute accuracies of the fine-tuned models. The bottom panel reports accuracies of the merged models, normalized to their corresponding fine-tuned baselines (\%).
	}
	\label{tab:six_nli_benchmarks_rank16_test}
	\vspace{-4pt}
	\renewcommand{\arraystretch}{1.0}
	\centering
	\scriptsize
	\setlength{\tabcolsep}{2pt}
	\resizebox{0.99\linewidth}{!}{
		\begin{tabular}{l|cccccc|c}
			\toprule
			\textbf{Method}                        & \textbf{MNLI} & \textbf{QNLI} & \textbf{SNLI} & \textbf{RTE} & \textbf{SICK} & \textbf{SciTail} & \textbf{Avg}  \\
			\midrule
			\multicolumn{8}{c}{\textit{Per-task absolute accuracy (\%)}}                                                                                             \\
			\midrule
			Finetuned                              & 90.8          & 94.9          & 91.8          & 87.0         & 90.9          & 94.8             & 91.7          \\
			\midrule
			\multicolumn{8}{c}{\textit{Merged models (normalized to finetuned baselines, \%)}}                                                                       \\
			\midrule
			TA~\cite{ilharco2022editing}           & 92.8          & 86.8          & 93.3          & 93.6         & 83.8          & 95.0             & 90.9          \\
			TIES~\cite{yadav2023ties}              & 94.3          & 88.8          & 90.8          & 89.8         & 86.6          & 94.4             & 90.8          \\
			DARE-TIES~\cite{yu2024language}        & 94.3          & 88.8          & 90.8          & 89.8         & 86.6          & 94.4             & 90.8          \\
			SVD~\cite{tang2025lora}                & 95.4          & 88.6          & 92.8          & 93.7         & 80.0          & 92.4             & 90.5          \\
			Linear~\cite{peft}                     & 96.0          & 86.1          & 88.1          & 94.4         & 83.8          & 96.1             & 90.8          \\
			KnOTS-TIES~\cite{stoica2025knots}      & 92.0          & 82.0          & 94.9          & 92.1         & 80.2          & 95.3             & 89.4          \\
			KnOTS-DARE-TIES~\cite{stoica2025knots} & 91.8          & 82.3          & 94.9          & 91.3         & 80.1          & 95.4             & 89.3          \\
			LoRA-LEGO~\cite{zhao2025loralego}      & 87.7          & 85.1          & 90.9          & 92.1         & 86.1          & 95.6             & 89.6          \\
			EMR-Merging~\cite{huang2024emr}        & 96.2          & 88.0          & 94.2          & 92.9         & 76.6          & 94.5             & 90.4          \\
			FR-Merging~\cite{zheng2025free}        & 95.5          & 88.7          & 93.1          & 93.6         & 80.3          & 92.6             & 90.6          \\
			RobustMerge~\cite{zeng2025parameter}   & 94.3          & 88.1          & 93.7          & 93.6         & 83.0          & 94.5             & 91.2          \\
			AdaMerging~\cite{yang2023adamerging}   & 94.3          & 84.8          & 92.5          & 92.1         & 89.2          & 84.8             & 89.6          \\
			\rowcolor[gray]{0.9}
			NSC (Ours)                             & 94.9          & 88.3          & 92.8          & 91.3         & 91.2          & 95.1             & \textbf{92.3} \\
			\bottomrule
		\end{tabular}}
	\vspace{-10pt}
\end{table}

\vspace{2pt}
\noindent
\noindent\textbf{Implementation Details.}
We use AdamW \cite{loshchilov2018decoupled} with a learning rate of 0.001. The merging weights $\lambda$ are initialized to 0.4.
For heterogeneous vision tasks, we optimize merging weights $\lambda$ for 100 iterations with a batch size of 32.
For LLM and VLM experiments, we use a batch size of 2 and 1, respectively, for 500 iterations with a learning rate of 0.0003.
We compare against ten baselines and follow the experimental protocols reported in their original papers.
For TA \cite{ilharco2022editing}, TIES \cite{yadav2023ties}, and KnOTS \cite{stoica2025knots}, we set a global merge scale $\lambda$ and tune it on a small validation split using validation loss.
For AdaMerging \cite{yang2023adamerging}, all merging coefficients $\{\lambda_k^{\,\ell}\}_{k=1,\ell=1}^{K,L}$ are learned with the entropy surrogate, where $N$ is the number of tasks and $L$ is the number of layers.
Further details about the algorithm and fine-tuning details are provided in~\cref{append:finetune} of supplementary material.

\begin{table*}[t]
	\centering
	\caption{
		Per-task performance of merged \textsc{LLaVA-1.5-7B} models across six multi-modal benchmarks. Each model fine-tunes the Vicuna-7B language backbone with LoRA \citep{hu2022lora} (rank = 16), while fully fine-tuning the multi-modal projector. The upper block reports absolute accuracies of fine-tuned baselines; the lower block presents merged models’ results normalized to their corresponding baselines (\%).
	}
	\vspace{-5pt}
	\label{tab:six_llava_long2_benchmarks_rank16}
	\renewcommand{\arraystretch}{0.9}
	\setlength{\tabcolsep}{8pt}
	\scriptsize
	\resizebox{0.90\textwidth}{!}{
		\begin{tabular}{l|cccccc|c}
			\toprule
			\textbf{Method}                                     & \textbf{IconQA} & \textbf{VizWiz$_\text{val}$} & \textbf{ChartQA} & \textbf{DocVQA$_\text{val}$} & \textbf{COCO} & \textbf{Flickr30k} & \textbf{Avg}  \\
			\midrule
			\multicolumn{8}{c}{\textit{Per-task absolute score}}                                                                                                                                                        \\
			\midrule
			Zero-Shot                                           & 17.9            & 55.2                         & 18.2             & 24.3                         & 109.5         & 79.2               & --            \\
			Finetuned                                           & 67.8            & 69.3                         & 39.0             & 40.7                         & 130.5         & 91.3               & --            \\
			Avg. generated tokens                               & 2.1             & 2.8                          & 4.7              & 6.7                          & 12.5          & 14.2               & 7.17          \\
			\midrule
			\multicolumn{8}{c}{\textit{Merged models (normalized to fine-tuned baselines, \%)}}                                                                                                                         \\
			\midrule
			TA~\cite{ilharco2022editing}                        & 56.8            & 83.2                         & 73.0             & 85.2                         & 92.9          & 97.6               & 81.4          \\
			TIES~\cite{yadav2023ties}                           & 61.2            & 81.6                         & 76.1             & 82.3                         & 92.9          & 96.7               & 81.8          \\
			DARE-TIES~\cite{yu2024language}                     & 57.4            & 84.4                         & 72.8             & 84.4                         & 92.5          & 95.8               & 81.2          \\
			SVD~\cite{tang2025lora}                             & 56.2            & 82.7                         & 72.9             & 85.5                         & 92.6          & 96.6               & 81.1          \\
			Linear~\cite{peft}                                  & 52.6            & 89.8                         & 70.0             & 79.2                         & 87.0          & 95.9               & 79.1          \\
			KnOTS-TIES~\cite{stoica2025knots}                   & 49.8            & 84.1                         & 70.4             & 85.5                         & 92.3          & 97.2               & 79.9          \\
			KnOTS-DARE-TIES~\cite{stoica2025knots}              & 52.4            & 86.1                         & 72.6             & 85.3                         & 93.0          & 96.3               & 80.9          \\
			LoRA-LEGO~\cite{zhao2025loralego}                   & 53.1            & 86.5                         & 68.4             & 82.5                         & 91.5          & 96.6               & 79.8          \\
			EMR-Merging~\cite{huang2024emr}                     & 84.0            & 49.2                         & 65.6             & 84.3                         & 92.0          & 97.1               & 78.7          \\
			FR-Merging~\cite{zheng2025free}                     & 87.2            & 46.3                         & 61.5             & 76.2                         & 90.2          & 94.8               & 76.0          \\
			RobustMerge~\cite{zeng2025parameter}                & 82.3            & 58.2                         & 71.1             & 82.1                         & 93.6          & 96.9               & 80.7          \\
			AdaMerging~\cite{yang2023adamerging} (Single Token) & 64.7            & 76.0                         & 76.4             & 82.1                         & 87.5          & 98.4               & 80.9          \\
			AdaMerging~\cite{yang2023adamerging} (Full Token)   & 68.7            & 76.2                         & 77.9             & 85.8                         & 91.1          & 94.4               & 82.4          \\
			\rowcolor[gray]{0.9}
			NSC (Ours)                                          & 59.7            & 82.9                         & 78.1             & 87.1                         & 91.7          & 96.8               & \textbf{82.7} \\
			\bottomrule
		\end{tabular}}
	\vspace{-6pt}
\end{table*}

\subsection{Experimental Results}

\subsubsection{Results across 20 heterogeneous vision tasks}
To test that our method works well on both classification and regression, 
we test on 20 tasks from NYUD-v2 (4), PASCAL-Context (5), and Taskonomy (11), encompassing diverse dense predictions like semantic segmentation, depth, and normals.
\Cref{tab:nyud_pascal_taskonomy} reports normalized per-task scores for a ViT-B backbone. Gradient-free methods collapse on dense prediction in NYUD-v2 and PASCAL-Context. For example, TA, TIES, and DARE-TIES achieve only about 1–2\% on NYUD-v2 semantic segmentation, whereas NSC reaches 85.1\%. On PASCAL-Context, vanilla methods stay near 20–21\% on Parts and about 33–38\% on Saliency, while NSC attains 76.7\% and 88.7\% respectively.
SVD and Linear improve several Taskonomy metrics but suffer pronounced drops on others (e.g., NYUD-v2 Depth), lowering their overall averages.
This indicates that prior gradient-free approaches have clear limitations when merging models for heterogeneous tasks. NSC provides the most balanced outcome. It stays near parity on all eleven Taskonomy tasks (about 99.6–102.2\%) and markedly outperforms alternatives on NYUD-v2 and PASCAL-Context. This balance yields the best average normalized score at 92.0\% across twenty tasks, indicating strong retention of task-specific performance in heterogeneous settings.

\subsubsection{Per-task evaluation across six NLI tasks}
To assess the scalability of merging methods, we report per-task normalized accuracies on six NLI benchmarks. In \cref{tab:six_nli_benchmarks_rank16_test}, we fine-tune \textsc{LLaMA-3 8B} with LoRA adapters of rank 16 applied to the query and value projections.
Gradient-free baselines show limited ability to preserve task-specific information at this scale.
AdaMerging uses a pseudo-entropy surrogate by treating token-level probabilities as a label-free signal, yet it underperforms both gradient-free baselines and our method, which indicates that entropy-based objectives are not uniformly effective for LLMs. In contrast, NSC achieves consistent gains across all tasks and attains the best average normalized accuracy of 92.3\%, demonstrating the effectiveness of our approach on language benchmarks.

\subsubsection{Per-task evaluation across six VLM tasks}
\label{subsubsec:vlm_result}
We evaluate the scalability of NSC merging for sequence generation on six vision–language benchmarks in \cref{tab:six_llava_long2_benchmarks_rank16}. VLMs such as LLaVA~\cite{liu2023visual, liu2024improved} decode outputs token by token. Classification prompts often terminate after a single token, whereas captioning tasks need long sequences. AdaMerging must compute a entropy loss at each generated token, so cost grows with sequence length. To reflect this, we consider two AdaMerging settings as baselines: \textit{Single Token}, which updates with the entropy of only the first generated token, and \textit{Full Token}, which updates with sequence-level entropy over all generated tokens. NSC leverages the LoRA projection matrix’s null-space ratio, a structural signal independent of token logits, so even a single generated token provides a sufficient gradient signal. Its computational cost therefore matches the \textit{Single Token} setting. NSC consistently outperforms AdaMerging (\textit{Single Token}) and matches or slightly exceeds \textit{Full Token} across tasks. Given the reported sequence lengths, \textit{Full Token} requires on average about seven times more gradient updates than NSC to reach similar accuracy. A detailed analysis of computational cost is provided in \cref{ablation:cost}.

\begin{table*}[t]
	\caption{Experiments on generalization to unseen tasks, including additional baselines. Merging results on NYUD-v2, PASCAL-Context, and Taskonomy with a ViT-B backbone. Fine-tuned rows report per-task absolute performance. For each baseline, we report performance normalized (\%) to the corresponding fine-tuned score, and we also report averages over seen tasks, unseen tasks, and all tasks.} \vspace{-5pt}
	\centering
	\renewcommand\arraystretch{1.25}
	\resizebox{0.99\textwidth}{!}{
		\setlength{\tabcolsep}{2pt}
		\begin{tabular}{l|
				cc|cc|cccccc|c|
				cc|ccc|ccccc|c|c}
			\toprule
			 & \multicolumn{10}{c|}{\textbf{Seen}}   & \multirow{3}{*}{\shortstack{\textbf{Avg}                                                                                                                                                                                                                                                                                                                                            \\\textbf{(Seen)}}}
			 & \multicolumn{10}{c|}{\textbf{Unseen}} & \multirow{3}{*}{\shortstack{\textbf{Avg}                                                                                                                                                                                                                                                                                                                                            \\\textbf{(Unseen)}}} & \multirow{3}{*}{\textbf{Avg}} \\
			\cmidrule(lr){2-11}\cmidrule(lr){13-22}
			 & \multicolumn{2}{c|}{\textbf{NYUD-v2}} & \multicolumn{2}{c|}{\textbf{PASCAL-Context}} & \multicolumn{6}{c|}{\textbf{Taskonomy}}
			 &                                       & \multicolumn{2}{c|}{\textbf{NYUD-v2}}        & \multicolumn{3}{c|}{\textbf{PASCAL-Context}} & \multicolumn{5}{c|}{\textbf{Taskonomy}} &                   &                                                                                                                                                                                                                         \\
			\cmidrule(lr){2-3}\cmidrule(lr){4-5}\cmidrule(lr){6-11}\cmidrule(lr){13-14}\cmidrule(lr){15-17}\cmidrule(lr){18-22}
			\multirow{2}{*}{\textbf{Method}}
			 & Semseg                                & Edge                                         & Parts                                        & Sal                                     & DZ                & ET              & K3              & N               & R               & S2.5
			 &                                       & Depth                                        & Normal                                       & Semseg                                  & Normal            & Edge            & DE              & EO              & K2              & C                 & S2
			 &                                       &                                                                                                                                                                                                                                                                                                                                                                                     \\
			 & mIoU $\uparrow$                       & L1 $\uparrow$                                & mIoU $\uparrow$                              & mIoU $\uparrow$                         & L1 $\downarrow$   & L1 $\downarrow$ & L1 $\downarrow$ & L1 $\downarrow$ & L1 $\downarrow$ & L1 $\downarrow$
			 &                                       & RMSE $\downarrow$                            & Mean $\downarrow$                            & mIoU $\uparrow$                         & Mean $\downarrow$ & L1 $\uparrow$   & L1 $\downarrow$ & L1 $\downarrow$ & L1 $\downarrow$ & RMSE $\downarrow$ & L1 $\downarrow$
			 &                                       &                                                                                                                                                                                                                                                                                                                                                                                     \\
			\midrule
			\multicolumn{23}{c}{\textit{Finetuned performance (absolute)}}                                                                                                                                                                                                                                                                                                                                                                 \\ \cmidrule(lr){1-24}
			Finetuned
			 & 37.66                                 & 0.051                                        & 54.77                                        & 80.00                                   & 0.0160            & 0.1713          & 0.0820          & 0.2169          & 0.1357          & 0.1435
			 & --
			 & 0.657                                 & 25.98                                        & 70.07                                        & 18.36                                   & 0.046             & 0.016           & 0.101           & 0.162           & 0.710           & 0.170
			 & --                                    & --                                                                                                                                                                                                                                                                                                                                                                                  \\
			\midrule
			\multicolumn{23}{c}{\textit{Merged models, normalized to finetuned (\%)}}                                                                                                                                                                                                                                                                                                                                                      \\ \cmidrule(lr){1-24}
			TA~\cite{ilharco2022editing}
			 & 1.3                                   & 102.0                                        & 20.6                                         & 30.8                                    & 91.9              & 120.9           & 115.0           & 100.8           & 134.7           & 106.1             & 82.4            & 6.6  & 43.5 & 4.7  & 47.0 & 104.3 & 95.1  & 103.7 & 132.3 & 118.6 & 111.7 & 76.8          & 79.6          \\

			TIES~\cite{yadav2023ties}
			 & 1.3                                   & 104.2                                        & 20.6                                         & 28.7                                    & 99.6              & 111.4           & 105.2           & 100.6           & 127.9           & 105.7             & 80.5            & 9.6  & 49.0 & 4.7  & 51.4 & 104.3 & 103.5 & 98.9  & 123.2 & 113.2 & 105.4 & 76.3          & 78.4          \\

			DARE-TIES~\cite{yu2024language}
			 & 1.2                                   & 105.3                                        & 20.6                                         & 33.6                                    & 89.4              & 111.4           & 118.3           & 100.5           & 117.9           & 100.9             & 79.9            & 7.3  & 49.8 & 4.7  & 53.5 & 104.3 & 88.1  & 106.4 & 128.6 & 122.0 & 105.9 & 77.1          & 78.5          \\

			SVD~\cite{tang2025lora}
			 & 1.4                                   & 105.1                                        & 20.6                                         & 32.6                                    & 96.8              & 107.2           & 106.2           & 100.2           & 110.4           & 100.0             & 78.1            & 14.3 & 49.7 & 4.7  & 52.9 & 104.3 & 98.2  & 103.3 & 110.8 & 107.1 & 105.3 & 75.1          & 76.6          \\

			Linear~\cite{peft}
			 & 0.4                                   & 73.4                                         & 20.3                                         & 29.0                                    & 117.7             & 226.8           & 95.6            & 99.6            & 124.3           & 25.9              & 81.3            & 1.5  & 29.1 & 0.1  & 23.5 & 102.0 & 130.3 & 87.9  & 190.4 & 95.6  & 28.9  & 68.9          & 75.1          \\

			KnOTS-TIES~\cite{stoica2025knots}
			 & 1.3                                   & 80.3                                         & 20.6                                         & 28.5                                    & 96.9              & 119.9           & 108.7           & 100.4           & 116.0           & 99.1              & 77.2            & 5.2  & 43.5 & 4.7  & 44.5 & 104.3 & 99.5  & 104.3 & 121.7 & 110.0 & 105.7 & 74.3          & 75.8          \\

			KnOTS-DARE-TIES~\cite{stoica2025knots}
			 & 1.2                                   & 104.2                                        & 20.6                                         & 28.5                                    & 101.6             & 115.9           & 118.3           & 100.4           & 103.0           & 85.8              & 77.9            & 4.8  & 47.9 & 4.7  & 48.2 & 104.3 & 97.5  & 105.5 & 117.8 & 106.2 & 122.8 & 76.0          & 77.0          \\

			LoRA-LEGO~\cite{zhao2025loralego}
			 & 1.2                                   & 104.7                                        & 20.6                                         & 33.5                                    & 97.2              & 102.1           & 105.5           & 100.2           & 109.7           & 98.5              & 77.3            & 9.0  & 48.9 & 4.7  & 53.0 & 104.3 & 98.8  & 102.2 & 106.9 & 104.8 & 102.8 & 73.5          & 75.4          \\

			EMR-Merging~\cite{huang2024emr}
			 & 1.3                                   & 88.3                                         & 20.6                                         & 36.4                                    & 94.3              & 104.4           & 106.9           & 100.2           & 108.9           & 98.6              & 76.0            & 15.9 & 53.1 & 4.7  & 56.9 & 104.1 & 95.2  & 103.5 & 109.2 & 107.3 & 108.2 & 75.8          & 75.9          \\

			FR-Merging~\cite{zheng2025free}
			 & 1.3                                   & 104.0                                        & 20.6                                         & 31.4                                    & 92.9              & 121.0           & 113.0           & 100.8           & 134.9           & 106.2             & 82.6            & 6.8  & 45.1 & 4.7  & 47.3 & 104.3 & 96.3  & 101.4 & 130.5 & 117.2 & 110.5 & 76.4          & 79.5          \\

			RobustMerge~\cite{zeng2025parameter}
			 & 76.5                                  & 100.1                                        & 73.6                                         & 85.7                                    & 101.0             & 100.6           & 100.6           & 100.0           & 99.7            & 99.8              & 93.8            & 37.2 & 63.4 & 83.1 & 73.7 & 100.6 & 100.0 & 102.2 & 100.5 & 100.2 & 99.1  & 86.0          & 89.9          \\

			\rowcolor[gray]{0.9}
			NSC (Ours)
			 & 84.6                                  & 100.3                                        & 76.8                                         & 87.3                                    & 100.9             & 100.6           & 100.8           & 100.0           & 99.7            & 99.7              & \textbf{95.1}   & 40.8 & 65.0 & 85.4 & 76.0 & 101.2 & 100.0 & 102.3 & 100.6 & 100.3 & 99.3  & \textbf{87.1} & \textbf{91.1} \\
			\bottomrule
		\end{tabular}}
	\label{tab:nyud_pascal_taskonomy_seen_unseen_grouped}
	\vspace{-7pt}
\end{table*}

\subsubsection{Generalization to unseen tasks}
We evaluate generalization in a seen–unseen protocol inspired by AdaMerging~\cite{yang2023adamerging}. Prior work considered eight classification tasks. We expand the setting to twenty heterogeneous vision tasks that include regression. We select ten seen tasks and ten unseen tasks distributed across NYUD-v2, PASCAL-Context, and Taskonomy, and report per-task scores along with averages over seen, unseen, and all tasks.

\Cref{tab:nyud_pascal_taskonomy_seen_unseen_grouped} summarizes the results. Gradient-free methods struggle to preserve performance on unseen objectives and dense prediction. Their averages remain modest on both splits (e.g., TA 82.4\% seen and 76.8\% unseen, TIES 80.5\% seen and 76.3\% unseen). LoRA-aware methods can post strong performances on subsets of Taskonomy but remain unstable when moving to NYUD-v2 and PASCAL-Context. Linear shows pronounced volatility and the lowest unseen average among LoRA-aware baselines (68.9\%), and SVD improves some Taskonomy tasks yet drops on NYUD-v2 Depth and PASCAL-Context Normal, which pulls down its averages (78.1\% seen and 75.1\% unseen).
NSC delivers consistently balanced gains. It achieves the best averages on both splits, with 95.1\% on seen tasks and 87.1\% on unseen tasks, and the highest overall average of 91.1\%.

\subsection{Ablation Study}
\label{subsec:ablation}

\begin{table}[t]
	\centering
	\caption{Ablation on where the NSC objective is applied: target module and number of blocks. Evaluated on CLIP (ViT-B/32) across eight image classification tasks.}  \vspace{-5pt}
	\label{tab:projection-blocks}
	\setlength{\tabcolsep}{2.8pt} 
	\renewcommand{\arraystretch}{0.9}
	\resizebox{0.47\textwidth}{!}{
		\begin{tabular}{lcccc}
			\toprule
			\multirow{2}{*}{\textbf{Targeted Projection Matrix}} &
			\multicolumn{4}{c}{\textbf{Number of Activated Transformer Blocks}}                                       \\
			\cmidrule(lr){2-5}
			                                                     & \textbf{12} & \textbf{6} & \textbf{3} & \textbf{1} \\
			\midrule
			QKVO                                                 & 83.7        & 84.4       & 83.9       & 83.0       \\
			KVO                                                  & 84.0        & 84.6       & 84.0       & 83.0       \\
			VO                                                   & 84.3        & 84.7       & 84.1       & 82.6       \\
			O                                                    & 84.5        & 84.9       & 84.6       & 82.6       \\
			\bottomrule
			\vspace{-10pt}
		\end{tabular}%
	}
	\vspace{-10pt}
\end{table}

\vspace{2pt}
\noindent\textbf{Robustness to target module and block count.}
To ensure the NSC objective \cref{eq:nsc_merged} influences all LoRA adapters, it is intuitive to place it toward the back of the network so that gradients flow through every upstream adapter at least once. \cref{tab:projection-blocks} ablates the attention module targeted by NSC (Q, K, V, O and their subsets) and the number of activated blocks. Targeting the output projection O gives a slight edge across depths, but the spread across QKVO, KVO, and VO is small, indicating that NSC is robust to the module choice. Under a coverage constraint that always includes the last portion with a LoRA adapter, activating a small set of mid to late blocks is sufficient.
Performance peaks at 84.9 with O on six blocks and remains strong at 84.6 with three blocks, which we adopt as the default for a better accuracy–compute trade off. Using only one block degrades accuracy, while activating all twelve is unnecessary. In practice,
applying NSC to O on a few back half blocks is both effective and efficient.

\vspace{2pt}
\noindent\textbf{Computational cost and memory.}
\label{ablation:cost}
In \cref{tab:cvpr_time_efficiency}, we compare wall-clock time and peak memory of NSC and baselines on \textsc{LLaVA} across six tasks using 500 validation samples.
We include gradient-free methods (TA, TIES, KnOTS–TIES) and the gradient-based AdaMerging.
Gradient-free methods merge parameters and tune scales on validation splits, with TIES and KnOTS–TIES add preprocessing for sign-conflict or SVD of task vectors.
In contrast, both NSC merging and AdaMerging optimizes merge coefficients via a label-free objective without repeated evaluations.
While gradient-free methods consume less GPU memory, they require longer validation due to multiple next-token predictions.
Gradient-based approaches use slightly more memory, but with LoRA adapters the difference is minor compared to full fine-tuned model merging, keeping gradient-based optimization practical.
NSC matches AdaMerging (Single Token) in cost, with only negligible overhead from projection formation in \cref{eq:null_ratio_simple}.
Applying entropy minimization to all tokens causes AdaMerging’s optimization time to scale with sequence length.
Overall, NSC achieves a favorable performance–efficiency balance and surpasses prior baselines, as shown in \cref{tab:six_llava_long2_benchmarks_rank16}.

\begin{table}[t]
	\centering
	\caption{
		Runtime and memory efficiency. Comparison of preparation, optimization, and validation time (minutes) and memory usage across gradient-free and gradient-based methods.
		Validation time \emph{excludes} grid-search overhead for gradient-free methods.
	}
	\vspace{-5pt}
	\footnotesize
	\setlength{\tabcolsep}{4pt}
	\renewcommand{\arraystretch}{1.15}
	\resizebox{0.99\linewidth}{!}{
		\begin{tabular}{l|ccc|ccc|cc}
			\toprule
			\textbf{Method}                                     &
			\multicolumn{3}{c}{\textbf{Requirements}}           &
			\textbf{Prep.}                                      & \textbf{Opt.}              & \textbf{Val.} &
			\textbf{\textit{Total}}                             & \textbf{\textit{GPU Mem.}}                                                            \\
			                                                    & \textbf{Prep.}             & \textbf{Opt.} & \textbf{Val.} &
			(min)                                               & (min)                      & (min)         & (min)         & (GB)                     \\
			\midrule
			\emph{Gradient-Free}                                &                            &               &               &               &  &  &  & \\
			TA~\cite{ilharco2022editing}                        &
			\cvprxmark                                          & \cvprxmark                 & \cvprcheck    &
			--                                                  & --                         & 16.6          & 16.6          & \textbf{14.4}            \\
			TIES~\cite{yadav2023ties}                           &
			\cvprcheck                                          & \cvprxmark                 & \cvprcheck    &
			4.5                                                 & --                         & 16.6          & 20.1          & \textbf{14.4}            \\
			KnOTS-TIES~\cite{stoica2025knots}                   &
			\cvprcheck                                          & \cvprxmark                 & \cvprcheck    &
			6.8                                                 & --                         & 16.6          & 23.4          & \textbf{14.4}            \\
			\midrule
			\emph{Gradient-Based}                               &                            &               &               &               &  &  &  & \\
			AdaMerging~\cite{yang2023adamerging} (Single Token) &
			\cvprxmark                                          & \cvprcheck                 & \cvprxmark    &
			--                                                  & 13.3                       & --            & \textbf{13.3} & 18.0                     \\
			AdaMerging (Full Token)                             &
			\cvprxmark                                          & \cvprcheck                 & \cvprxmark    &
			--                                                  & 104.0                      & --            & 104.0         & 18.0                     \\
			\rowcolor[gray]{0.9}
			\textbf{NSC (Ours)}                                 &
			\cvprcheck                                          & \cvprcheck                 & \cvprxmark    &
			$\approx$0.0                                        & 13.3                       & --            & \textbf{13.3} & 17.9                     \\
			\bottomrule
		\end{tabular}
	}
	\label{tab:cvpr_time_efficiency}
	\vspace{-10pt}
\end{table}

Additional analyses including computational cost and memory are presented in~\cref{append:more_exp} of the supplement.

\section{Conclusion}
We studied LoRA merging, where multiple adapters are combined to build a single model.
During LoRA fine-tuning, the down-projection factor $\bm A$ in $\Delta \bm W = \bm B \bm A$ defines an adapter subspace.
We define the proportion of the activation that is suppressed by this projection as the null-space ratio.
Empirically, this ratio decreases as task performance improves, providing a simple label-free signal for learning merge weights.
Building on this insight, we introduce Null-Space Compression (NSC) Merging,
a computationally efficient approach for combining multiple LoRA adapters. Across twenty heterogeneous vision tasks, six NLI tasks, and six vision–language tasks, NSC achieves balanced, state-of-the-art performance.

\vspace{4pt}
\noindent\textbf{Acknowledgements}
This research was supported by the Challengeable Future Defense Technology Research and Development Program through the Agency For Defense Development(ADD) funded by the Defense Acquisition Program Administration(DAPA) in 2026(No.915102201)

{
		\small
		\bibliographystyle{ieeenat_fullname}
		\bibliography{main}
	}


\clearpage
\setcounter{page}{1}
\maketitlesupplementary

\setcounter{section}{0}
\renewcommand{\thesection}{\Alph{section}}

\setcounter{figure}{2}
\setcounter{table}{6}

\section{Experimental Settings}
\subsection{Datasets and Tasks}
\label{append:datasets}

\subsubsection{Heterogeneous Vision Tasks}
We evaluate 20 dense prediction vision tasks with a ViT-based multi-task model across three datasets spanning indoor and outdoor scenes.

\begin{itemize}[leftmargin=*]
	\item \textbf{NYUD-v2~\citep{silberman2012indoor}.} An indoor RGB-D dataset of scenes such as living rooms, offices, and kitchens. We use four dense prediction tasks: depth estimation, semantic segmentation, surface normal prediction, and edge detection. The official split provides 795 training images and 654 test images but no validation set. In our experiments, we keep the 795 training images and randomly divide the original test split into 327 validation and 327 test images.
	\item \textbf{PASCAL-Context~\citep{mottaghi2014role}.} An extension of PASCAL VOC 2010 with dense pixel-wise annotations for natural scenes containing everyday objects. We use five tasks: semantic segmentation, human parts estimation, saliency estimation, surface normal prediction, and edge detection. The dataset provides 10,103 images with an official split of 4,998 training and 5,105 validation images but no test set. We keep the 4,998 training images and split the original validation set into 2,607 validation and 2,498 test images following~\cite{liu2018deep, wan2020super}.
	\item \textbf{Taskonomy~\citep{zamir2018taskonomy}.} A large-scale indoor dataset with diverse geometric and semantic tasks. We construct a tiny subset for training and validation by sampling 259,747 training and 35,774 validation images from the standard Taskonomy buildings, and use images from the Ihlen building as the test set with 9,007 images. We adopt eleven dense prediction tasks: Depth Euclidean (DE), Depth Z-buffer (DZ), Edge Texture (ET), Keypoints 2D (K2), Keypoints 3D (K3), Normal (N), Principal Curvature (C), Reshading (R), Segment Unsup2D (S2), and Segment Unsup2.5D (S2.5).
\end{itemize}

\subsubsection{NLI Tasks}
\label{append:nli_datasets}
We evaluate six sentence-pair classification tasks with LLaMA-3 8B.

\begin{itemize}[leftmargin=*]
	\item \textbf{QNLI~\cite{wang_glue_2018}.} Question–answering reformulated as sentence-pair classification. Each example is a question and a candidate sentence, and the label indicates whether the sentence contains the answer (binary).
	\item \textbf{MNLI~\cite{williams_broad-coverage_2018}.} A multi-genre collection of premise–hypothesis pairs for natural language inference. The goal is to predict one of three labels: \emph{entailment}, \emph{contradiction}, or \emph{neutral}.
	\item \textbf{SNLI~\cite{bowman_large_2015}.} A large corpus of human-written premise–hypothesis pairs supporting standard NLI with the same three labels as MNLI. The dataset contains roughly 570k examples.
	\item \textbf{RTE~\cite{dagan_pascal_2006,bar-haim_second_2006,giampiccolo_third_2007,dang_fourth_2009,bentivogli_fifth_nodate}.} A suite of textual entailment benchmarks consolidated in GLUE as binary \emph{entailment} versus \emph{not-entailment}. Examples are drawn from sources such as news and Wikipedia.
	\item \textbf{SICK~\cite{marelli_sick_2014}.} Sentence pairs constructed from image and video captions. Each pair has a three-way entailment label \emph{entailment}, \emph{contradiction}, or \emph{neutral} and a semantic relatedness score on a 1–5 scale. The dataset includes about 10k pairs.
	\item \textbf{SciTail~\cite{khot_scitail_2018}.} Entailment derived from science multiple-choice questions and retrieved web sentences. Each premise–hypothesis pair is labeled \emph{entails} or \emph{neutral} with 27,026 total examples \mbox{10,101} \emph{entails} and \mbox{16,925} \emph{neutral}.
\end{itemize}

\begin{table}[t]
	\caption{Vision--Language Tasks Setup.}
	\vspace{-5pt}
	\centering
	\renewcommand{\arraystretch}{1.10}
	\setlength{\tabcolsep}{10pt}
	\resizebox{0.99\linewidth}{!}{
		\begin{tabular}{l p{0.80\linewidth}}
			\toprule
			\textbf{Dataset}                  & \textbf{Provided Prompt} \\
			\midrule

			\textbf{IconQA}, \textbf{VizWiz}  &
			Answer the question using a single word.                     \\

			\midrule

			\textbf{ChartQA}, \textbf{DocVQA} &
			Answer the question using a single word or short phrase.     \\

			\midrule

			\textbf{COCO}, \textbf{Flickr30K} &
			Provide a one-sentence caption for the provided image.       \\

			\bottomrule
		\end{tabular}}
	\label{tab:vlm_finetune}
	\vspace{-5pt}
\end{table}

\subsubsection{VLM Tasks}
\label{append:vlm_datasets}

We evaluate our vision–language setup using LLaVA-1.5-7B~\cite{liu2023visual, liu2024improved}.
Our evaluation primarily follows the protocol described in~\citet{liu2024improved}.
For datasets not included in that work, we adopt the evaluation procedure from~\citet{zhang2024lmmsevalrealitycheckevaluation}.

\begin{itemize}
	\item \textbf{IconQA~\cite{lu2021iconqa}.}
	      A diagram-based question–answering benchmark that evaluates visual reasoning over abstract and textbook-style images,
	      emphasizing symbolic understanding and multi-step reasoning.

	\item \textbf{VizWiz~\cite{gurari2018vizwiz}.}
	      A visual question–answering dataset consisting of images captured by visually impaired users.
	      The questions are often open-ended and exhibit real-world visual and linguistic noise.
	      We randomly split the validation and test sets in a 2:8 ratio.

	\item \textbf{ChartQA~\cite{masry2022chartqa}.}
	      A visual question–answering dataset requiring numerical and semantic reasoning over data presented in scientific charts and plots.

	\item \textbf{DocVQA~\cite{mathew2021docvqa}.}
	      A document-based visual question–answering dataset designed for understanding scanned documents, forms, and invoices.
	      We randomly split the validation and test sets in a 2:8 ratio.

	\item \textbf{COCO~\cite{lin2014microsoft}.}
	      A large-scale dataset for image captioning and object detection, containing diverse natural images with dense object annotations and descriptive captions.
	      We use the captioning task with the train/val/test split protocol proposed by~\citet{Karpathy_2015_CVPR}.

	\item \textbf{Flickr30k~\cite{flickr30k}.}
	      A benchmark for image–text retrieval and captioning, consisting of 31K real-world images paired with multiple human-written captions.
	      We use the captioning task with the train/val/test split protocol proposed by~\citet{Karpathy_2015_CVPR}.
\end{itemize}

\subsection{Metrics}
\label{append:metrics}
We evaluate each task using metrics that are commonly adopted for that task in the literature. For semantic segmentation, saliency estimation, and human-part segmentation, we use mean Intersection-over-Union (mIoU). Surface normal prediction is evaluated by the mean angular error between predicted and ground-truth normals. Depth estimation is measured with root mean squared error (RMSE). Edge detection is assessed using the optimal dataset scale F-measure (ods-F).
For the Taskonomy, principal curvature uses RMSE, and the remaining tasks use L1 distance, following \cite{chen2023mod}.
For NLI and VLM benchmarks, we report classification accuracy for sequence-classification and visual question answering tasks~\cite{liu2024improved,zhang2024lmmsevalrealitycheckevaluation}.
For image captioning, we report CIDEr.

\subsection{Implementation Details}
\label{append:imple}
For the heterogeneous vision setting, we optimize our learnable merging variants using the Adam~\cite{kingma2014adam} with a learning rate of 1e-4 for 100 iterations.
For conventional learning-free merging methods~\cite{ilharco2022editing,yadav2023ties,yu2024language,stoica2025knots,huang2024emr,zheng2025free,zeng2025parameter}, we choose the global scaling coefficient of the merged parameters by grid search on the validation loss. We first sweep the scale with a step size of 0.1, then run a finer sweep with a step size of 0.01 around the best scale found in the coarse search. Several baselines require additional hyperparameters beyond the global scale. For TIES~\cite{yadav2023ties} and KnOTS-TIES~\cite{stoica2025knots}, we keep the top 20\% of parameters by magnitude. DARE-TIES~\cite{yu2024language} and KnOTS-DARE-TIES~\cite{stoica2025knots} use a drop rate of 0.9 when discarding parameters. FR-Merging~\cite{zheng2025free} uses a low-frequency ratio of 0.10 and a high-frequency ratio of 0.70 for frequency-based filtering. For RobustMerge~\cite{zeng2025parameter}, we follow the hyperparameter configuration reported in the original work.
For merging both LLMs and VLMs, we optimized the learnable merging coefficients using AdamW~\cite{loshchilov2018decoupled} with a learning rate of $3\times10^{-4}$ for 500 iterations. For LLM experiments, we used a per-task batch size of 2, whereas for VLMs the batch size was set to~1.

We merge the multimodal projector into a unified module. While learning-free baselines follow their respective merging protocols, gradient-based methods utilize learnable scalar coefficients for each projector. In NSC, these coefficients are optimized via the objective, allowing gradients calculated in the LLM to flow back and update them.

\subsection{Fine-tuning Details}
\label{append:finetune}

For the heterogeneous vision setting, we start from a ViT-B model pretrained on ImageNet-21k~\cite{deng2009imagenet} and fine-tune it separately on each task. We train with a batch size of 32 for 40,000 iterations using a learning rate of 2e-5, weight decay of 1e-6, and a polynomial learning-rate schedule. LoRA~\cite{hu2022lora} adapters with rank 16 are applied to the Q, K, V, and O projection modules of each attention layer.

For both the natural language inference and vision-language task settings, LLaMA~\cite{dubey2024llama} and LLaVA~\cite{liu2024improved,liu2023visual} were fine-tuned using LoRA with rank~16.
LoRA adapters were applied to the query and value projection matrices of the self-attention modules.
We used the AdamW optimizer~\cite{loshchilov2018decoupled} together with a cosine learning rate scheduler~\cite{loshchilov2017sgdr}, employing a warmup phase corresponding to 6\% of the total training steps.
The learning rate was set to 3e-5.

For LLaVA, the vision encoder was kept frozen during fine-tuning, while the multimodal projectors were fully fine-tuned. For visual question answering tasks, we fine-tuned the model for 5 epochs. For image captioning, we fine-tuned for one epoch and treated each caption associated with the same image as an independent training sample.

\section{Additional Related Work}
\label{append:more_related_work}

More recent works on model merging directly target VLMs and LLMs and propose strategies tailored to these architectures~\cite{zhu2025remedy,zeng2025parameter,chen2025bring,du2025adamms,lu2024twin,wei2025unifying,zeng2025parameter,zhou2024metagpt}. In particular, RobustMerge~\cite{zeng2025parameter} introduces a merging method for MLLMs that exploits directional robustness in a low rank space. \citet{zhang2025beyond} leverage the rotation symmetry of self attention layers, which substantially enlarges the equivalence set of transformer models compared to permutation based symmetries. Another line of work explicitly connects model merging with conventional multi task learning~\cite{yang2024surgeryv2,yang2024representation,shen2024efficient,wei2025modeling}, viewing merging as a mechanism for parameter sharing and representation consolidation across tasks. A complementary set of studies aims to localize task specific information in the parameters or to quantify interference within linear layers~\cite{davari2024model,tam2023merging,wang2024localizing,cheng2025whoever}, while \citet{marczak2025no} further decompose the parameter space into shared and task specific subspaces. MuDSC~\cite{xu2024training} also explores heterogeneous settings with diverse dense prediction tasks for evaluating merging performance. However, it is not directly comparable to our setting because it allows branch like architectures such as ZipIt~\cite{stoica2023zipit}, whereas we focus on a strictly shared backbone.

Several methods exploit low rank structures for merging~\cite{prabhakar2024lora,gargiulo2025task,panariello2025accurate}. In particular, \citet{panariello2025accurate} propose a core space that can be efficiently combined with existing merging baselines. Model merging has also been studied in continual learning scenarios~\cite{Dziadzio_2025_CVPR,tang2025merging,yang2025continual,wang2024lines,wang2025memoir,qiu2025null,lee2025interaction}, where merging is used to mitigate catastrophic forgetting and to accumulate knowledge across tasks over time. Another line of work leverages intermediate activations during inference or training and uses feature responses in specific layers to guide merging process~\cite{nobari2025activation,liu2025sens,yao2025activation,wu2025unlocking}. Different from other approach, which modifies the merging procedure itself, pre merging methods~\cite{zhang2025lori,tang2023parameter,zhang2025unraveling} focus on constructing checkpoints that are more amenable to merging, shifting the emphasis from the merging algorithm to the pre training and finetuning pipeline.

In the supplementary material, we further compare against more recent merging baselines~\cite{huang2024emr,zheng2025free,zeng2025parameter} to demonstrate the robustness of our algorithms.

\section{More Analysis on Null-Space Compression}

\subsection{Null-Space Ratio Computation}
\label{append:gram_inverse_null_space_ratio}

\paragraph{Null-Space Ratio with Gram-Inverse.}
We briefly derive Eq.~(\textcolor{cvprblue}{6}).
Let $\bm{A} \in \mathbb{R}^{r \times d}$ denote the LoRA down-projection matrix.
The row space of $\bm{A}$ is
\[
	\mathcal{R}(\bm{A}^\top)
	= \{ \bm{A}^\top \bm{u} : \bm{u} \in \mathbb{R}^r \}
	\subset \mathbb{R}^d,
\]
and its orthogonal complement is the null space
\[
	\mathcal{N}(\bm{A}) = \{ \bm{z} \in \mathbb{R}^d : \bm{A}\bm{z} = \bm{0} \}.
\]
The projection matrix onto $\mathcal{R}(\bm{A}^\top)$ is
\[
	\bm{P}
	= \bm{A}^\top (\bm{A}\bm{A}^\top)^{-1} \bm{A}
\]
and, the projection matrix onto the null space is
\[
	\bm{P}_{\mathrm{null}}
	= \bm{I} - \bm{P}
	= \bm{I} - \bm{A}^\top (\bm{A}\bm{A}^\top)^{-1} \bm{A}.
\]

For a feature vector $\bm{z} \in \mathbb{R}^d$, the null-space ratio is defined as
\[
	\omega(\bm{z})
	= \frac{\| \bm{P}_{\mathrm{null}} \bm{z} \|_2}{\|\bm{z}\|_2}.
\]
We compute the squared ratio:
\begin{align*}
	\omega(\bm{z})^2
	 & = \frac{\| (\bm{I} - \bm{P}) \bm{z} \|_2^2}{\|\bm{z}\|_2^2}          \\
	 & = \frac{\bm{z}^\top (\bm{I} - \bm{P})^\top (\bm{I} - \bm{P}) \bm{z}}
	{\|\bm{z}\|_2^2}.
\end{align*}
Because $\bm{P}$ is symmetric and idempotent,
\(
(\bm{I} - \bm{P})^\top(\bm{I} - \bm{P}) = \bm{I} - \bm{P}.
\)
Thus,
\begin{align*}
	\omega(\bm{z})^2
	 & = \frac{\bm{z}^\top (\bm{I} - \bm{P}) \bm{z}}{\|\bm{z}\|_2^2} \\
	 & = 1 - \frac{\bm{z}^\top \bm{P} \bm{z}}{\|\bm{z}\|_2^2}        \\
	 & = 1 - \frac{\bm{z}^\top
		\bm{A}^\top (\bm{A}\bm{A}^\top)^{-1} \bm{A} \bm{z}}
	{\|\bm{z}\|_2^2}.
\end{align*}
Taking the square root yields
\[
	\omega(\bm{z})
	= \sqrt{
		1 -
		\frac{
			\bm{z}^\top
			\bm{A}^\top (\bm{A}\bm{A}^\top)^{-1} \bm{A} \bm{z}
		}
		{\|\bm{z}\|_2^2}
	}.
\]
Replacing $\bm{A}$ with $\bm{A}_k$ recovers Eq.~(\textcolor{cvprblue}{6}).

In the perspective of a LoRA-equipped neural network, the term $\bm{A}\bm{z}$ is already computed during inference.
Therefore, rather than directly applying the full projection matrix, which would require constructing and multiplying by a dense $d \times d$ operator and is prohibitively expensive for modern LLMs and VLMs where $d$ easily reaches several thousands, we instead rely on a more efficient formulation based on the Gram-inverse.
Using the Gram-inverse $(\bm{A}_k \bm{A}_k^{\top})^{-1}$ yields a substantially more efficient formulation. All computations remain confined to the $r$-dimensional LoRA subspace, avoiding the need to store large projection matrices on GPU and reducing the complexity from $\mathcal{O}(d^2)$ to $\mathcal{O}(r^2)$.
Since the LoRA rank $r$ is typically tiny relative to $d$ (e.g., $r = 16$ while $d$ is in the thousands), this approach significantly reduces both memory usage and FLOPs. Consequently, the Gram-inverse formulation enables efficient evaluation of the null-space ratio during LoRA merging and analysis, making it suitable for large-scale models, as demonstrated in Sec.~\textcolor{cvprblue}{4.3}.

From a numerical standpoint, the Gram-inverse is also stable to compute. Because $\bm{A}_k \bm{A}_k^{\top}$ is an $r \times r$ symmetric positive (semi-)definite matrix, its inverse can be obtained efficiently using Cholesky factorization, which offers strong numerical stability and minimal memory overhead.

\begin{figure*}[t]
	\centering

	\begin{subfigure}{\linewidth}
		\centering
		\includegraphics[width=\linewidth]{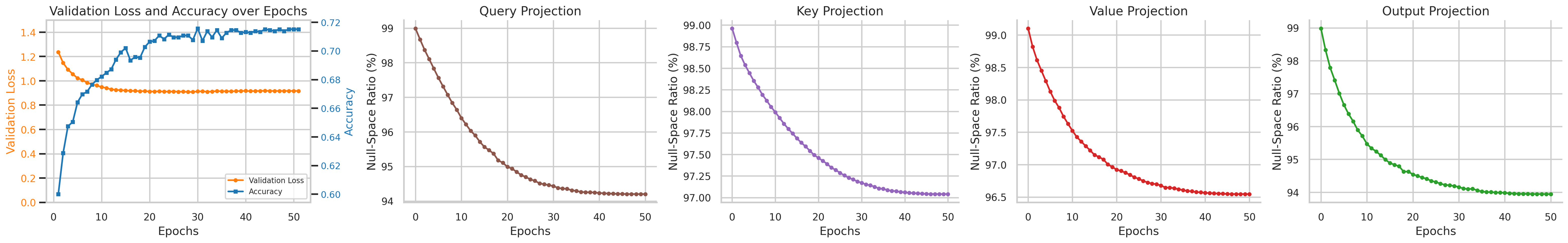}%
		\caption{Image Classification (Classification Task)}
		\label{subfig:cls_nsc_supp}
	\end{subfigure}

	\vspace{0.6em}

	\begin{subfigure}{\linewidth}
		\centering
		\includegraphics[width=\linewidth]{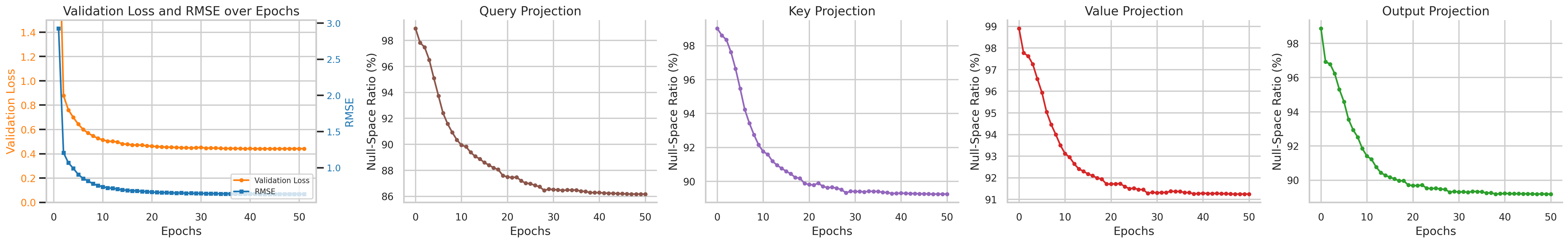}%
		\caption{Depth Estimation (Regression Task)}
		\label{subfig:depth_nsc_supp}
	\end{subfigure}

	\caption{
		Extended visualization of validation loss, task performance, and the null-space ratio during LoRA fine-tuning on
		(a) image classification and (b) depth estimation. We additionally show
		the null-space ratio trajectories of LoRA at query, key, value projection of self-attention module within a
		single transformer block, further illustrating the null-space compression phenomenon.
	}
	\label{fig:nsc_finetuning_supp}
	\vspace{-5pt}
\end{figure*}
\begin{figure}[t]
	\centering

	\begin{subfigure}{\linewidth}
		\centering
		\includegraphics[width=\linewidth]{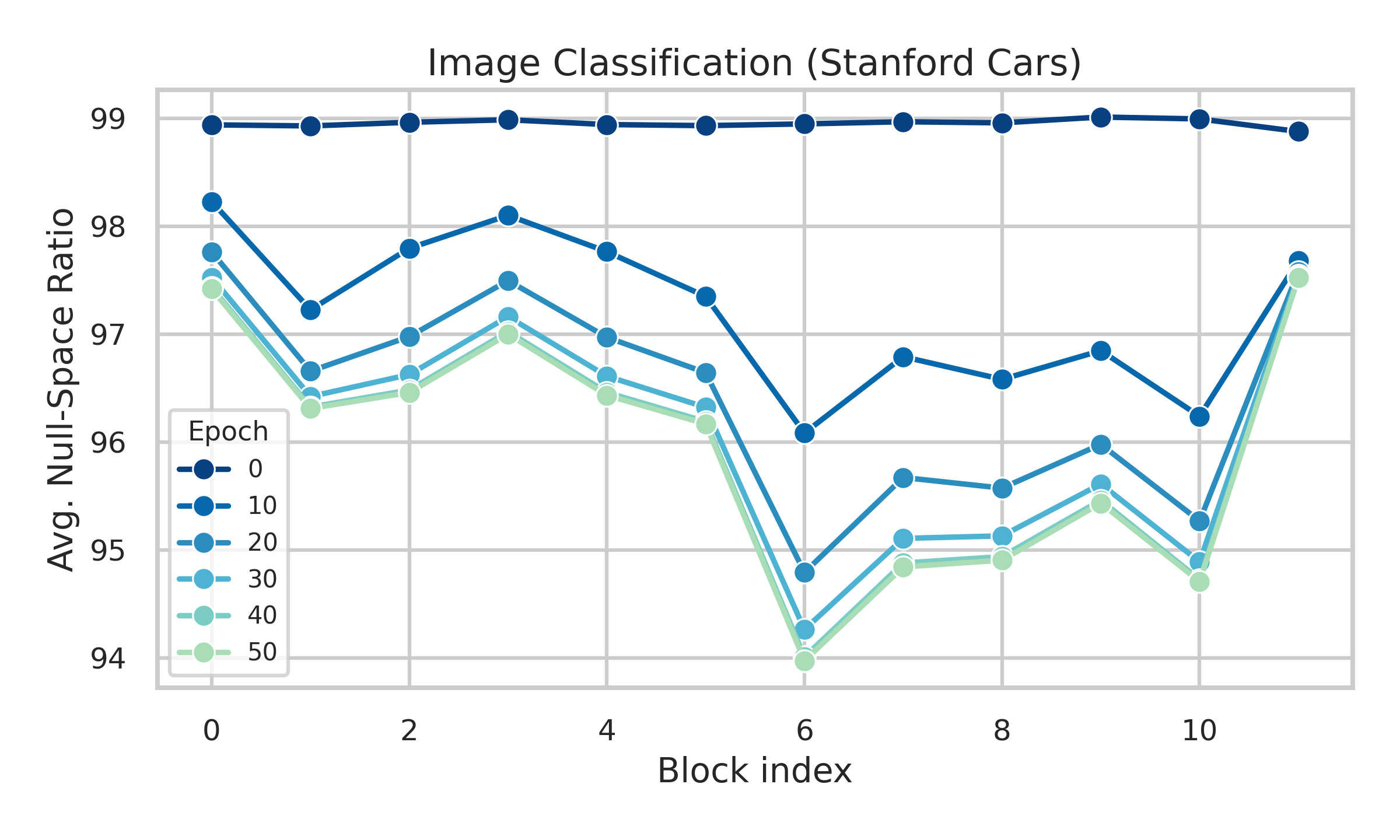}
		\caption{Image Classification (Classification Task)}
		\label{subfig:block_cls_nsc_supp}
	\end{subfigure}

	\begin{subfigure}{\linewidth}
		\centering
		\includegraphics[width=\linewidth]{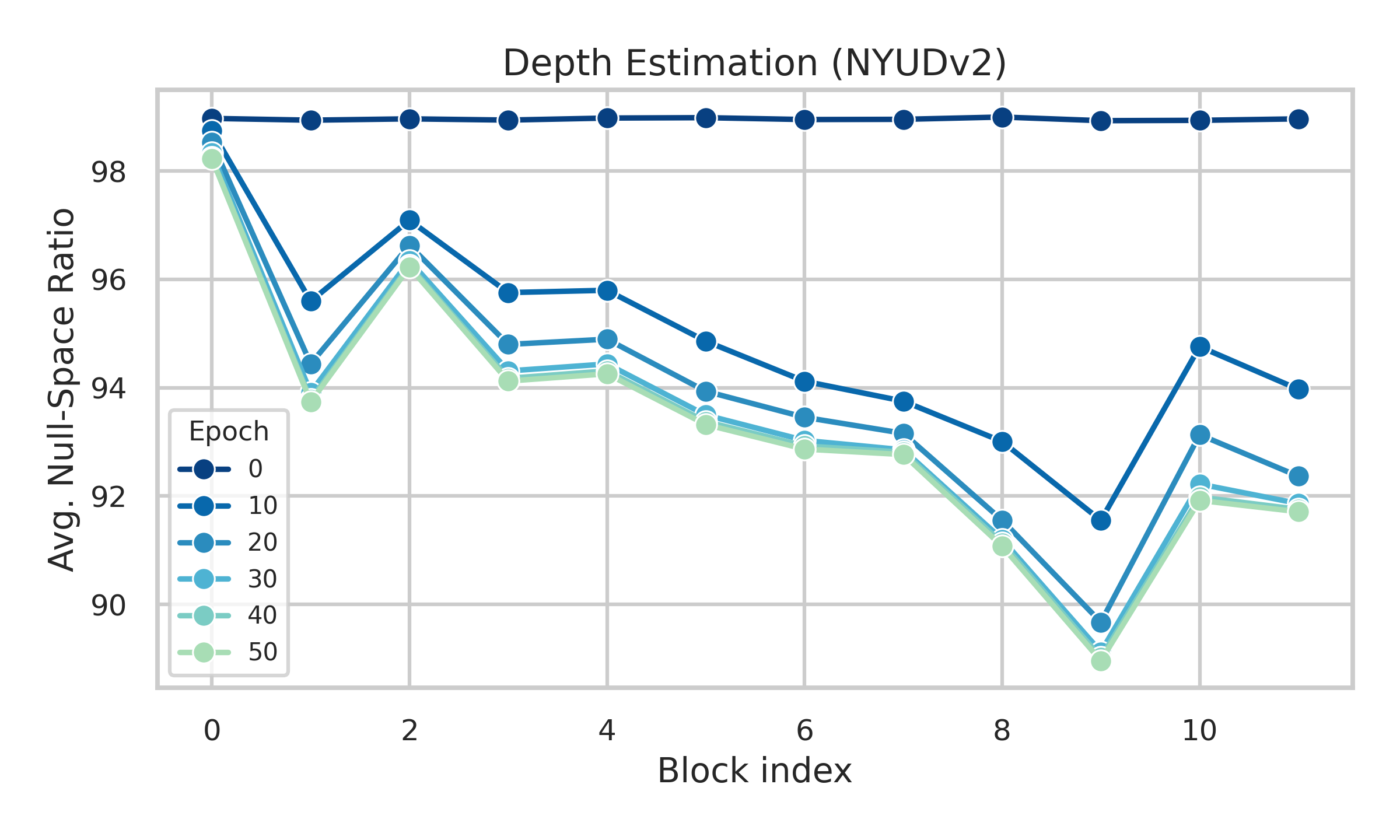}
		\caption{Depth Estimation (Regression Task)}
		\label{subfig:block_depth_nsc_supp}
	\end{subfigure}

	\caption{
		Null-space ratio of each transformer block during LoRA fine-tuning for (a) image classification and (b) depth estimation.
		Values are averaged across transformer blocks, showing that null-space compression consistently occurs throughout the model.
	}
	\label{fig:block_nsc_finetuning_supp}
	\vspace{-5pt}
\end{figure}

\paragraph{Impact on Adapter Strength.}

We analyze how the null-space ratio controls the effective magnitude of the LoRA update.
For simplicity, we omit the layer index.
Let $\Delta \bm W = \bm B_k \bm A_k$ be a LoRA update for task $k \in \{1, \dots, K\}$.

\begin{proposition}[Adapter effect lower bound]
	\label{prop:adapter_strength_lower_bound}
	For any LoRA update $\Delta \bm W = \bm B_k \bm A_k$, and $\bm{z} \in \mathbb{R}^d$, the following inequality holds:
	\begin{equation*}
		\|\bm{B}_k\bm{A}_k\bm{z}\|_2
		\;\ge\;
		C_k \sqrt{1 - \omega_k^2(\bm{z})} \;\|\bm{z}\|_2,
	\end{equation*}
	where $C_k = \sigma_{\min}(\bm{B}_k)\,\sigma_{\min}(\bm{A}_k)$.
\end{proposition}

\begin{proof}
	We begin with the standard singular value lower bound
	\[
		\|\bm{A}_k\bm{z}\|_2
		\;=\;
		\|\bm{A}_k\bm{P}\bm{z}\|_2
		\;\ge\;
		\sigma_{\min}(\bm{A}_k) \,
		\|\bm P \bm{z}\|_2,
	\]
	where $\bm P$ is the projection matrix onto a row space of $\bm A_k$.
	Applying the same inequality to $\bm{B}_k$ yields
	\[
		\|\bm{B}_k\bm{A}_k\bm{z}\|_2
		\;\ge\;
		\sigma_{\min}(\bm{B}_k)\,\|\bm{A}_k\bm{z}\|_2.
	\]
	Combining the inequalities gives
	\[
		\|\bm{B}_k\bm{A}_k\bm{z}\|_2
		\;\ge\;
		\sigma_{\min}(\bm{B}_k)\,\sigma_{\min}(\bm{A}_k)
		\;\|\bm P\bm{z}\|_2.
	\]
	Finally, using the definition of the null-space ratio,
	\[
		1 - \omega_k^2(\bm{z})
		=
		\frac{\|\bm P\bm{z}\|_2^2}{\|\bm{z}\|_2^2},
	\]
	we obtain the desired result.
\end{proof}

This bound shows that the strength of the LoRA update is fundamentally governed by how much of the input feature $\bm{z}$ lies inside the row space of the down-projection matrix $\bm{A}$.
The term $(1 - \omega_k^2(\bm{z}))$ quantifies this alignment: it is large when $\bm{z}$ contains a substantial row-space component and small when $\bm{z}$ is mostly aligned with the null space.
Because the lower bound grows proportionally to $\sqrt{1 - \omega_k^2(\bm{z})}$, any decrease in the null-space ratio necessarily increases the minimum possible magnitude of the LoRA update.
On multi-task settings, the lower null-space ratio guarantees higher lower-bound for the task-specific activation of the layer.
Consequently, the NSC objective implicitly matches task-specific activations without requiring explicit access to them and yielding stronger task-conditioned updates.

\subsection{Null-Space Compression during the Fine-tuning of LoRA Adapters}
\label{append:more_analysis_nsc_finetune}
We further analyze null-space compression during fine-tuning of LoRA-equipped models. Specifically, in \cref{fig:nsc_finetuning_supp} we report the null-space ratio for different parts of the self-attention module that were analyzed in the main paper. We track the null-space ratio over 50 epochs of LoRA fine-tuning for all self-attention projections (query, key, value, and output) within a representative transformer block, and we use the 9th block for this visualization. For both image classification and depth estimation, all projections exhibit a consistent decrease in null-space ratio over training. This trend shows that null-space compression is not confined to a particular projection but emerges across all components of the attention mechanism. Despite the small dimensionality of the LoRA subspace, we still observe a substantial reduction in null-space ratio across all LoRA-equipped modules.

To understand how null-space compression occurs across depth, we measure the null-space ratio for each transformer block during LoRA fine-tuning. \Cref{fig:block_nsc_finetuning_supp} reports the evolution of the null-space ratio over 50 epochs for all transformer blocks on both image classification and depth estimation. Across both tasks, all blocks start with relatively high null-space ratios and exhibit a gradual and consistent decrease as training progresses, indicating that null-space compression is not restricted to a few layers but occurs throughout the network. Together with the analysis in \cref{fig:nsc_finetuning_supp}, these results show that our objective can be used effectively across both the components and the depth of the attention module.

\section{More Experiments}
\label{append:more_exp}
\subsection{Extended Analysis of Main Results}

\noindent\textbf{Analysis on Heterogeneous Vision Tasks.}
As presented in \cref{tab:nyud_pascal_taskonomy} of the main paper, our comparison includes recent merging methods such as EMR-Merging~\cite{huang2024emr}, FR-Merging~\cite{zheng2025free}, and RobustMerge~\cite{zeng2025parameter}. Here, we provide a more detailed analysis of their performance across the twenty heterogeneous tasks. EMR-Merging degrades on dense prediction, yielding lower averages overall (77.0\% and 79.5\%). FR-Merging applies a Fourier-domain filter that emphasizes mid- to high-frequency weight components to reduce interference, which improves over vanilla baselines yet still trails the top methods (83.4\%). RobustMerge is the strongest among conventional baselines, showing balanced performance across tasks with notably better preservation of semantic segmentation and depth. NSC achieves the highest overall average (92.0\%) and leads on the hardest tasks (for example, 85.1\% on NYUD-v2 Semseg and 76.7\% on PASCAL-Context Parts), indicating more stable performance when objectives differ across classification and regression.

We observe that collapsed tasks in conventional baselines are predominantly high-level vision tasks (e.g., semantic segmentation), while low-level tasks (e.g., edge detection) tend to dominate the merged model.
We posit that \textit{inter-task affinity} serves as an implicit metric for this misalignment.
Following~\cite{fifty2021efficiently}, we define affinity as the relative reduction in the loss of task $j$ after a gradient update on task $i$: $\text{Affinity}_{i \rightarrow j} = 1 - \frac{\mathcal{L}_j(\theta - \eta \nabla \mathcal{L}_i)}{\mathcal{L}_j(\theta)}$.
Empirically, collapsed tasks show notably low affinity with the dominating tasks.
For instance, in PASCAL-Context, the affinity of Edge Detection towards Semantic Segmentation is only 0.34, whereas the affinity towards Surface Normals is much higher at 0.79.
This suggests that simple parameter averaging (TA, TIES) inherently fails when the dominant task introduces parameter shifts that suppress the activations required by other tasks.

While NSC does not eliminate the inherent dissimilarity between tasks, it prevents the catastrophic collapse stemming from signal vanishing.
Theoretically, as proven in~\cref{prop:adapter_strength_lower_bound}, minimizing the null-space ratio provides a mathematical guarantee on the lower bound of the adapter's signal strength ($\|\bm B_k \bm A_k z\|$).
In baselines, low affinity causes task signals to cancel out and fall into the null space.
In contrast, NSC explicitly optimizes coefficients to preserve these principal subspaces, serving as a safeguard to prevent feature collapse.

\begin{table*}[t]
	\caption{
		Ablation on the token position used to compute the NSC objective.
		We compare First-Token, Last-Token, and Full-Sequence scoring on six VLM benchmarks.
		We report average normalized performance and optimization time.
		The chosen design is shaded.
	}
	\vspace{-5pt}
	\centering
	\scriptsize
	\setlength{\tabcolsep}{2pt}
	\begin{tabularx}{0.99\linewidth}{D D D|S S S S S S|S S}
		\toprule
		\multicolumn{3}{c|}{\textbf{Token Position Strategy}} &
		\multicolumn{6}{c|}{\textbf{Dataset}}                 &
		\multicolumn{2}{c}{\textbf{Performance}}                                                                                                         \\
		\cmidrule(lr){1-3} \cmidrule(lr){4-11}
		\textbf{First-Token}                                  & \textbf{Last-Token} & \textbf{Full-Sequence} &
		\textbf{IconQA}                                       & \textbf{VizWiz}     & \textbf{ChartQA}       &
		\textbf{DocVQA}                                       & \textbf{COCO}       & \textbf{Flickr}        &
		\textbf{Avg}                                          & \textbf{Time (min)}                                                                      \\
		\midrule
		                                                      &                     & \checkmark             &
		59.5                                                  & 82.8                & 78.0                   & 87.0 & 91.5 & 97.3 & 82.7          & 83.7 \\
		                                                      & \checkmark          &                        &
		60.0                                                  & 82.9                & 77.7                   & 87.1 & 91.5 & 97.2 & \textbf{82.8} & 65.1 \\
		\rowcolor[gray]{0.9}
		\checkmark                                            &                     &                        &
		59.7                                                  & 82.9                & 78.1                   & 87.1 & 91.7 & 96.8 & 82.7          & 13.3 \\
		\bottomrule
	\end{tabularx}
	\label{tab:vlm_pos_abl}
\end{table*}
\begin{table*}[t]
	\centering
	\caption{
		Ablation on the input modality used to compute the NSC objective.
		We compare text-only, vision-only, and joint vision–language scoring on six VLM benchmarks.
		We report average normalized performance and optimization time.
		The chosen design is shaded.
	}
	\vspace{-5pt}
	\scriptsize
	\setlength{\tabcolsep}{2pt}
	\begin{tabularx}{0.99\linewidth}{D D|S S S S S S|S S}
		\toprule
		\multicolumn{2}{c|}{\textbf{Input Modality}} &
		\multicolumn{6}{c|}{\textbf{Dataset}}        &
		\multicolumn{2}{c}{\textbf{Performance}}                                                                                          \\
		\cmidrule(lr){1-2} \cmidrule(lr){3-10}
		\textbf{Text}                                & \textbf{Image}      &
		\textbf{IconQA}                              & \textbf{VizWiz}     & \textbf{ChartQA} &
		\textbf{DocVQA}                              & \textbf{COCO}       & \textbf{Flickr}  &
		\textbf{Avg}                                 & \textbf{Time (min)}                                                                \\
		\midrule
		\checkmark                                   &                     &
		56.6                                         & 83.8                & 74.3             & 83.9 & 93.3 & 97.7 & 81.6          & 6.0  \\
		                                             & \checkmark          &
		60.2                                         & 82.4                & 78.4             & 87.0 & 91.4 & 97.0 & \textbf{82.7} & 13.1 \\
		\rowcolor[gray]{0.9}
		\checkmark                                   & \checkmark          &
		59.7                                         & 82.9                & 78.1             & 87.1 & 91.7 & 96.8 & \textbf{82.7} & 13.3 \\
		\bottomrule
	\end{tabularx}
	\label{tab:vlm_mod_abl}
\end{table*}

\vspace{2pt}
\noindent\textbf{Analysis on VLM Benchmarks.}
Building on the results in \cref{tab:six_llava_long2_benchmarks_rank16}, we further analyze per-task performance on the six multi-modal benchmarks. Among the newly added baselines, EMR-Merging~\cite{huang2024emr} and FR-Merging~\cite{zheng2025free} both achieve strong performance on benchmarks such as MNLI and SNLI (e.g., 96.2 and 94.2 for EMR-Merging on MNLI and SNLI, respectively), but exhibit noticeable degradation on SICK, dropping to 76.6 and 80.3. RobustMerge~\cite{zeng2025parameter} yields the best average among all prior baselines, with consistently high scores across MNLI, SNLI, RTE, and SciTail, yet still underperforms our approach. NSC attains the highest overall average of 92.3, demonstrating that our method works well on NLI datasets compared to recent merging baselines.

\vspace{2pt}
\noindent\textbf{Detailed Analysis on VLM Benchmarks.}
Building on the results in \cref{tab:six_llava_long2_benchmarks_rank16}, we further analyze per-task performance on the six multi-modal benchmarks, where all merged results are normalized to their corresponding fine-tuned baselines. Among the recent methods, EMR-Merging~\cite{huang2024emr} and FR-Merging~\cite{zheng2025free} show large gains on IconQA (84.0 and 87.2, respectively), but suffer severe degradation on VizWiz (49.2 and 46.3), leading to lower overall averages of 78.7 and 76.0. As a result, both methods underperform traditional parameter-space baselines such as TA and TIES, which reach 81.4 and 81.8 on average, and also fall behind adaptive approaches like AdaMerging~\cite{yang2023adamerging}, which achieves 80.9 with single-token adaptation and 82.4 with full-token adaptation. NSC attains the highest overall average of 82.7, while maintaining strong performance, indicating that NSC extends well to multi-modal VLM settings and improves over recent baselines.

\vspace{2pt}
\noindent\textbf{Analysis on Generalization to Unseen Tasks.}
\Cref{tab:nyud_pascal_taskonomy_seen_unseen_grouped} evaluates generalization on ten seen and ten unseen tasks. Expanding on the main text, we observe that the overall trends mirror those in \cref{tab:nyud_pascal_taskonomy}. NSC attains the best averages on both splits, with 95.1\% on seen, 87.1\% on unseen, and 91.1\% overall, indicating strong retention on seen tasks and robust transfer to tasks without checkpoints. EMR-Merging similarly degrade on heterogeneous dense tasks. FR-Merging improves some Taskonomy metrics but remains unstable across datasets. RobustMerge is the most balanced among the recent baselines. Overall, NSC provides the most consistent gains across datasets and task types.

\vspace{2pt}
\noindent\textbf{Impact of token position and sequence length on NSC.}
\Cref{tab:vlm_pos_abl} ablates where in the generated sequence we compute the NSC objective, comparing First-Token, Last-Token, and Full-Sequence scoring on six VLM benchmarks. All three variants achieve very similar performance, with averages, indicating that NSC is largely insensitive to the exact token position used. We attribute this robustness to the fact that NSC operates on the geometry of the LoRA parameters (through the null-space ratio) rather than relying on a particular token’s semantics. Moreover, using the full sequence does not yield noticeable gains over single-token variants. In contrast, First-Token scoring avoids waiting for the full response to be generated before backpropagation, making NSC substantially more efficient in practice. We therefore adopt First-Token as our default design.

\begin{table}[t]
	\centering
	\caption{Detailed compute cost breakdown of AdaMerging and NSC on VLM tasks. All times are in seconds and measured on a single NVIDIA A6000 GPU.}
	\label{tab:compute_cost}
	\scriptsize
	\setlength{\tabcolsep}{3pt}
	\renewcommand{\arraystretch}{1.0}
	\resizebox{0.90\linewidth}{!}{
		\begin{tabular}{l|cc}
			\toprule
			\textbf{Process Step}            & \textbf{AdaMerging} & \textbf{NSC (Ours)} \\
			\midrule
			Pre-computation (Gram-Inverse)   & \cvprxmark          & 0.7                 \\
			Forward Pass                     & 396.7               & 407.2               \\
			Entropy/NSC Calculation          & 1.1                 & 4.1                 \\
			Backpropagation + Optimizer Step & 393.4               & 386.7               \\
			\bottomrule
		\end{tabular}}
	\vspace{-5pt}
\end{table}

\vspace{2pt}
\noindent\textbf{Effect of input modality in NSC.}
\Cref{tab:vlm_mod_abl} studies which input modality is used to compute the NSC objective on six VLM benchmarks, comparing text-only, image-only, and joint vision–language inputs. Text-only scoring performs reasonably well but lags behind the configurations that use image features. Image-only and joint scoring both reach an average of 82.7, and consistently outperform text-only, showing that visual information is more critical for guiding NSC in multi-modal settings. Based on this observation, we adopt the joint vision–language configuration as our design.

\vspace{2pt}
\noindent\textbf{Compute Accounting.}
In addition to \cref{tab:cvpr_time_efficiency}, we provide a more detailed breakdown of the optimization costs of AdaMerging and NSC in \cref{tab:compute_cost}.
The Gram-Inverse pre-computation step for NSC takes approximately 0.7 seconds, which is a one-time cost at the beginning of optimization and does not affect the per-iteration cost during training.
The forward pass and backpropagation steps are comparable between the two methods, with NSC being slightly faster in backpropagation due to more efficient gradient updates.
The NSC calculation itself takes about 4.1 seconds per iteration, which is more expensive than the entropy calculation in AdaMerging (1.1 seconds), but this additional cost is relatively small compared to the overall iteration time.
Therefore, while NSC introduces some overhead due to its more complex objective, it remains computationally feasible and efficient for practical use in large-scale models.
All results are based on VLM tasks, with additional cost analysis reported in \cref{tab:cvpr_time_efficiency}, where our method demonstrates faster preparation compared to SVD-based baselines requiring pre-computation.

\begin{table}[t]
	\centering
	\caption{Ablation on where the NSC objective is applied: target module and number of blocks evaluated on 20 heterogeneous vision tasks.}  \vspace{-5pt}
	\label{tab:projection-blocks_heterogeneous}
	\setlength{\tabcolsep}{2.8pt}
	\renewcommand{\arraystretch}{0.9}
	\resizebox{0.99\linewidth}{!}{%
		\begin{tabular}{lCCCC}
			\toprule
			\multirow{2}{*}{\textbf{Targeted Projection Matrix}} &
			\multicolumn{4}{c}{\textbf{Number of Activated Transformer Blocks}}                                       \\
			\cmidrule(lr){2-5}
			                                                     & \textbf{12} & \textbf{6} & \textbf{3} & \textbf{1} \\
			\midrule
			QKVO                                                 & 92.0        & 92.0       & 92.0       & 91.9       \\
			KVO                                                  & 91.6        & 91.6       & 91.6       & 91.4       \\
			VO                                                   & 92.0        & 92.0       & 92.0       & 91.6       \\
			O                                                    & 91.7        & 91.7       & 91.8       & 91.7       \\
			\bottomrule
		\end{tabular}%
	}
	\vspace{-5pt}
\end{table}

\vspace{2pt}
\noindent\textbf{Robustness to target module and block count.}
\Cref{tab:projection-blocks_heterogeneous} ablates where the NSC objective is applied in the attention stack and how many transformer blocks are activated on the 20 heterogeneous vision tasks. We vary the targeted projection matrix across QKVO, KVO, VO, and O, and sweep the number of activated blocks from all 12 layers down to only the last block. Across all configurations, the average performance remains clustered around 92.0, showing that NSC is highly robust to the choice of projection matrix. Targeting VO or QKVO slightly better performance, but the gap relative to using only O is within 0.3, which is negligible.

\begin{table}[t]
	\centering
	\scriptsize
	\caption{Optimization stability of gradient-based merging methods across 5 random seeds.}
	\label{tab:variance}
	\renewcommand{\arraystretch}{1.0}
	\setlength{\tabcolsep}{3pt}
	\resizebox{0.85\linewidth}{!}{
		\begin{tabular}{llc}
			\toprule
			\textbf{Benchmark}         & \textbf{Method} & \textbf{Avg. Norm. Acc.} (mean $\pm$ std) \\
			\midrule
			\multirow{2}{*}{LLM Tasks} & AdaMerging      & 90.08 $\pm$ 0.44                          \\
			                           & NSC (Ours)      & 92.25 $\pm$ 0.08                          \\
			\bottomrule
		\end{tabular}}
\end{table}

\begin{table}[t]
	\centering
	\caption{Impact of unlabeled sample size on performance of NSC on LLM tasks.}
	\renewcommand{\arraystretch}{0.99}
	\setlength{\tabcolsep}{3pt}
	\resizebox{0.99\linewidth}{!}{
		\label{tab:data_eff}
		\begin{tabular}{l|cccc}
			\toprule
			\textbf{\# Samples per Task} & \textbf{1}       & \textbf{10}      & \textbf{100}     & \textbf{Full Data} \\
			\midrule
			Avg. Norm. Acc.              & 91.88 $\pm$ 0.35 & 92.19 $\pm$ 0.16 & 92.25 $\pm$ 0.12 & 92.25 $\pm$ 0.08   \\
			\bottomrule
		\end{tabular}}
\end{table}

\vspace{2pt}
\noindent\textbf{Variance Analysis \& Data Efficiency.}
Since gradient-free baselines (e.g., TA~\cite{ilharco2022editing}, TIES~\cite{yadav2023ties}) are deterministic, we evaluate optimization stability by comparing against AdaMerging~\cite{yang2023adamerging} on LLM tasks.
As shown in \cref{tab:variance}, NSC demonstrates substantially lower variance across random seeds, achieving a standard deviation of 0.08 compared to 0.44 for AdaMerging.
This indicates that NSC yields more consistent performance and is less sensitive to initialization and stochastic effects during optimization.

We further analyze data efficiency in \cref{tab:data_eff}.
Although performance slightly decreases with fewer unlabeled samples, NSC retains near-peak accuracy even under extremely limited data regimes, achieving competitive results with as few as 10 samples per task.

\vspace{2pt}
\noindent\textbf{Zero-shot Image Classification Benchmarks.}
Following task arithmetic (TA)~\cite{ilharco2022editing}, we also evaluate our method on merging eight image-classification models based on CLIP/ViT-B-32~\cite{radford2021learning}.
All models are fine-tuned using LoRA with rank 16. We report performance on eight widely used classification benchmarks, SUN397~\cite{xiao_sun_2010}, Cars~\cite{krause20133d}, RESISC45~\cite{cheng_remote_2017}, EuroSAT~\cite{helber_eurosat_2019}, SVHN~\cite{netzer_reading_nodate}, GTSRB~\cite{stallkamp_german_2011}, MNIST~\cite{lecun_mnist_2005}, and DTD~\cite{cimpoi_describing_2014}.
As shown in \Cref{tab:clip_short}, NSC does not always outperform AdaMerging~\cite{yang2023adamerging} in this CLIP-based classification setting. On these traditional, classification-centric benchmarks, NSC consistently improves over learning-free baselines such as TA and TIES~\cite{ilharco2022editing,yadav2023ties}, achieving higher averaged normalized accuracy across the eight datasets. However, NSC lags behind AdaMerging, which is specifically designed with entropy-based objectives. These results suggest that NSC plays a complementary role to conventional entropy-driven merging.
In particular, NSC is most beneficial in regimes where entropy cannot be computed, such as regression tasks, or where computing or optimizing entropy is not efficient or effective, such as large-scale LLM and VLM settings, while entropy-based methods remain preferable on standard classification benchmarks.

\begin{table}[t]
	\centering
	\setlength{\tabcolsep}{2pt}
	\caption{
		Performance comparison across visual classification tasks. Absolute accuracies (top), and
		normalized accuracies of merged models (bottom).
	}
	\resizebox{0.99\linewidth}{!}{
		\begin{tabular}{l|c c c c c c c c|c}
			\toprule
			\multirow{2}{*}{\textbf{Method}}     &
			\multicolumn{9}{c}{\textbf{Dataset}}                                                                                                                                                          \\
			\cmidrule(lr){2-10}
			                                     & \textbf{Cars} & \textbf{DTD} & \textbf{EuroSAT} & \textbf{GTSRB} & \textbf{MNIST} & \textbf{RESISC45} & \textbf{SUN397} & \textbf{SVHN} & \textbf{Avg} \\
			\midrule
			\midrule
			Finetuned                            & 70.7          & 67.6         & 98.5             & 97.6           & 99.4           & 91.6              & 72.0            & 95.7          & 86.6         \\
			\midrule
			\midrule
			TA~\cite{ilharco2022editing}         & 87.8          & 78.6         & 62.9             & 70.5           & 92.2           & 78.8              & 92.4            & 85.9          & 81.1         \\
			TIES~\cite{yadav2023ties}            & 90.6          & 80.4         & 78.8             & 66.1           & 91.6           & 81.2              & 92.5            & 82.5          & 82.9         \\
			\midrule
			AdaMerging~\cite{yang2023adamerging} & 90.1          & 76.8         & 87.8             & 86.0           & 94.7           & 85.6              & 90.8            & 82.8          & 86.8         \\
			NSC (Ours)                           & 89.4          & 78.7         & 80.3             & 74.4           & 93.4           & 82.6              & 92.1            & 86.0          & 84.6         \\
			\bottomrule
		\end{tabular}
	}
	\label{tab:clip_short}
\end{table}

\end{document}